\documentclass{article}

\usepackage[final,main]{neurips_2026}

\usepackage[utf8]{inputenc}
\usepackage[T1]{fontenc}
\usepackage{url}
\usepackage{booktabs}
\usepackage{microtype}
\usepackage{graphicx}
\usepackage{caption}
\usepackage{subcaption}
\usepackage{arydshln}
\usepackage{multirow}
\usepackage{rotating}
\usepackage{siunitx}
\usepackage{makecell}
\usepackage{bm}
\usepackage{amsmath}
\usepackage{amssymb}
\usepackage{mathtools}
\usepackage{amsthm}
\usepackage{gensymb}
\usepackage{enumitem}
\usepackage{algorithm}
\usepackage{algorithmic}
\usepackage{xcolor}
\usepackage{colortbl}
\usepackage{placeins}
\usepackage[pagebackref=true,breaklinks=true,colorlinks=true,urlcolor=magenta,bookmarks=false,citecolor=blue]{hyperref}
\usepackage{xspace}
\usepackage[capitalize,noabbrev]{cleveref}

\newcommand{\return}{\textbf{return }}
\newcommand{\Comment}[1]{\hfill $\triangleright$ #1}

\makeatletter
\DeclareRobustCommand\onedot{\futurelet\@let@token\@onedot}
\def\@onedot{\ifx\@let@token.\else.\null\fi\xspace}
\def\eg{\emph{e.g}\onedot}

\makeatother

\newcommand{\dgs}{\textit{d}GS\xspace}
\newcommand{\dbs}{\textit{d}BS\xspace}
\newcommand{\dgso}{\textit{d}GS-O\xspace}
\newcommand{\dfamily}{\textit{d}-splatting\xspace}

\newcommand{\spacedcdashline}[1]{\noalign{\vspace{2pt}}\cdashline{#1}\noalign{\vspace{3pt}}}

\theoremstyle{plain}
\newtheorem{theorem}{Theorem}[section]
\newtheorem{proposition}[theorem]{Proposition}

\newtheorem{corollary}[theorem]{Corollary}
\theoremstyle{definition}

\theoremstyle{remark}

\AddToHook{env/theorem/begin}{\crefalias{section}{theorem}}
\AddToHook{env/proposition/begin}{\crefalias{section}{proposition}}
\AddToHook{env/lemma/begin}{\crefalias{section}{lemma}}
\AddToHook{env/corollary/begin}{\crefalias{section}{corollary}}
\AddToHook{env/definition/begin}{\crefalias{section}{definition}}
\AddToHook{env/assumption/begin}{\crefalias{section}{assumption}}
\AddToHook{env/remark/begin}{\crefalias{section}{remark}}

\title{Direct Conditional Parameterization for N-Dimensional Splatting}

\author{%
  \textbf{Zhongpai Gao\thanks{Corresponding author.} \quad Benjamin Planche \quad Van Nguyen Nguyen} \\
  \textbf{Meng Zheng \quad Anwesa Choudhuri \quad Arun Innanje} \\
  \textbf{Terrence Chen \quad Ziyan Wu} \\
  United Imaging Intelligence, Boston, MA, USA
}

\begin{document}

\maketitle

\vspace{-1.5em}
\begin{center}
\includegraphics[width=1\linewidth]{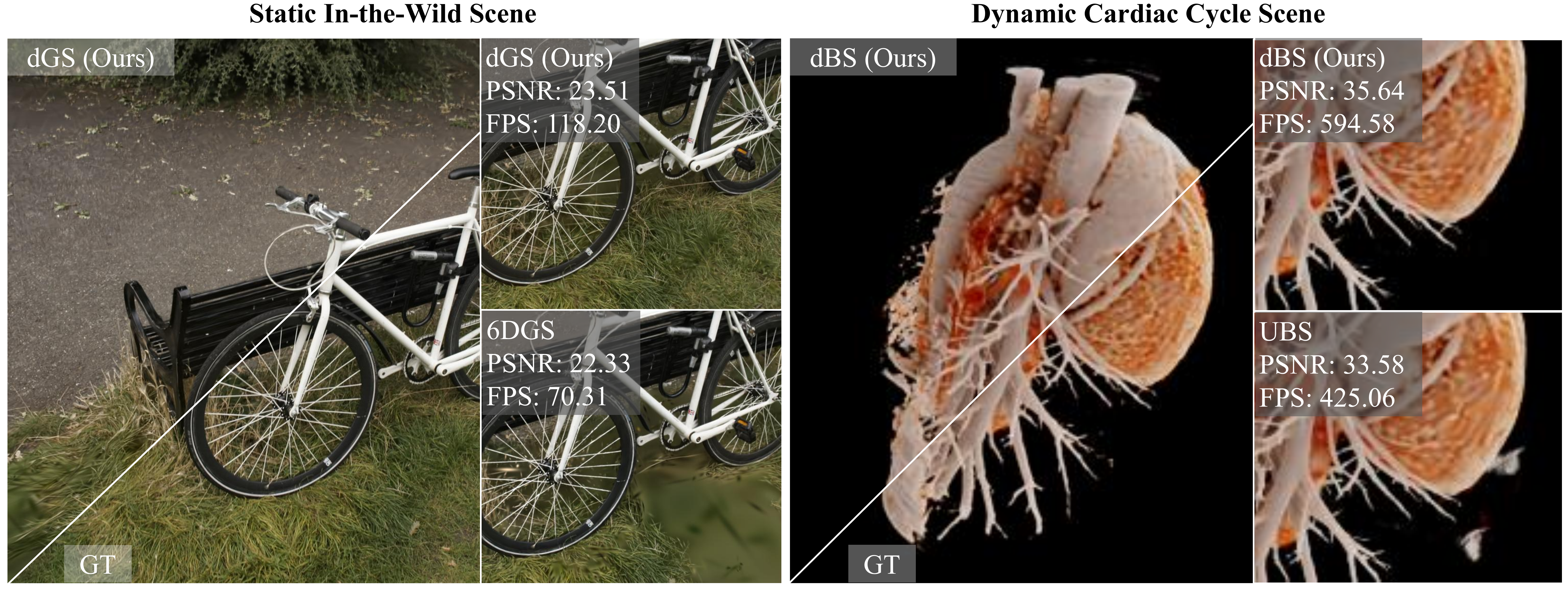}
\captionsetup{type=figure}
\captionof{figure}{\textbf{Rendering quality and efficiency comparison.} Our direct conditional parameterization yields two methods: \dgs{} (Gaussian kernel) outperforms its implicit-conditioning counterpart 6DGS \citep{gao20246dgs} on the static \texttt{bicycle} scene (left), and \dbs{} (Beta kernel) outperforms UBS \citep{liu2025universal} on the dynamic \texttt{heart1} scene (right). Both achieve higher quality and higher efficiency than their covariance-derived predecessors.}
\label{fig:teaser}
\end{center}
\vspace{-0.5em}

\begin{abstract}
   N-dimensional splatting extends 3D Gaussian Splatting with view direction and time to model view-dependent and dynamic effects. Each primitive is lifted into a joint distribution over position and conditioning variables, then conditionally sliced at each rendering query. N-DGS and UBS share this step across Gaussian and Beta kernels: recovering the conditional mean, opacity, and covariance requires per-primitive matrix inversions and regression products in every forward pass. We introduce direct conditional parameterization, replacing this covariance-derived readout with explicit per-effect parameters: a Cholesky precision factor for opacity and a factorized query-to-position regression operator with learnable per-dimension coupling. Unrestricted conditional coordinates recover a valid joint covariance; our factorization adds spatial-scale and coupling priors. Applied to both kernels, the parameterization yields \dgs{} (Gaussian) and \dbs{} (Beta). Across five static and dynamic benchmarks, they match or exceed their covariance-derived counterparts on quality, with gains up to +1.26 dB in the main MCMC evaluation. Production slicing is 6.9--7.7$\times$ faster than on-the-fly covariance-derived conditioning; complete-render profiling gives 1.25--1.91$\times$ speedups on four trained scenes. Direct conditioning provides a structured, kernel-agnostic alternative for N-dimensional splatting.
Project page: \url{https://gaozhongpai.github.io/dgs/}.
\end{abstract}

\section{Introduction}
\label{sec:introduction}
3D Gaussian Splatting (3DGS)~\citep{3dgs} achieves photorealistic novel-view synthesis at real-time speeds through explicit 3D Gaussian primitives. Static primitives, however, struggle with specular reflections, temporal dynamics, and view-dependent geometry.

Recent work extends 3DGS along two orthogonal axes. The first is \textit{kernel shape}: alternatives to the Gaussian kernel include Beta kernels~\citep{liu2025deformable,liu2025universal}, generalized exponentials~\citep{hamdi2024gesgeneralizedexponentialsplatting}, and Student-t distributions~\citep{zhu20253dstudentsplattingscooping}. The second is \textit{dimensionality}: N-dimensional Gaussians (N-DG)~\citep{diolatzis2024n}, refined as 6DGS~\citep{gao20246dgs} and 7DGS~\citep{gao20257dgs} (collectively N-DGS), lift each primitive into a joint distribution over 3D position and conditioning variables (view direction with $C{=}3$, view+time with $C{=}4$). At render time, the $(3{+}C)$-dimensional primitive is \textit{conditionally sliced} at the current query (a view direction or timestamp) using the conditional Gaussian formulas: the mean shifts, the opacity is modulated by a kernel evaluated at the query, and the spatial covariance receives a correction term. Universal Beta Splatting (UBS)~\citep{liu2025universal} combines the two axes by replacing the Gaussian opacity decay with an anisotropic Beta kernel while inheriting the same conditional-slicing structure. N-DGS and UBS together represent the leading N-dimensional splatting formulations across Gaussian and Beta kernels.

Despite their different kernels, N-DGS and UBS share the same conditional-slicing step, and this is where the per-primitive cost concentrates. Read-out of the three effects requires inverting the query block, constructing a regression matrix, and applying a covariance-correction product (Table~\ref{tab:comparison_overview}). These operations run once per primitive in every forward pass, so the total cost grows linearly with primitive count, which on large scenes like Mip-NeRF 360 reaches millions. Their ambient Cholesky coordinates also entangle the optimization of opacity precision, position regression, and conditional covariance, and conditioning can become numerically fragile when the query block is ill-conditioned.

\textbf{Our insight.} N-DGS uses ambient joint-covariance coordinates, although rendering needs only conditional spatial covariance, query precision, and query-to-position regression. Recovering them requires inversion and Schur complementation. Appendix~\ref{sec:conditional_coordinates} proves that this unrestricted operational tuple is a smooth, globally bijective coordinate chart for joint Gaussian covariances, with block-triangular reconstruction, an explicit inverse, and a nonvanishing coordinate-volume Jacobian throughout the SPD interior. \dgs{} directly parameterizes a structured factorization of these rendering-aligned coordinates. It therefore both removes expensive operations and separates position from opacity coupling.

We introduce \textbf{direct conditional parameterization} for N-dimensional splatting. Opacity uses a directly parameterized Cholesky precision factor, and position uses a factorized regression operator with learnable per-dimension coupling that decouples position from opacity by default and recovers full coupling on demand. Because the change touches only the conditioning step, the same parameterization is kernel-agnostic and carries across kernel families: we instantiate it as \dgs{} (Gaussian) for the direct counterpart of N-DGS and \dbs{} (Beta) for the direct counterpart of UBS. Pairing each direct method with its covariance-derived counterpart (\dgs{} vs.\ N-DGS, \dbs{} vs.\ UBS) isolates the contribution of direct conditioning from the choice of kernel.

Our contributions are:

\begin{itemize}[leftmargin=*]
\item \textbf{Direct conditional parameterization} for N-dimensional splatting: a kernel-agnostic alternative that explicitly parameterizes opacity precision and a factorized query-to-position regression operator. We prove that the underlying unrestricted operational tuple forms a smooth, globally bijective coordinate chart for joint Gaussian covariances, with an explicit inverse and a nonvanishing coordinate-volume Jacobian.

\item \textbf{Two kernel instantiations and matched-pair evaluation.} We develop the formulation for the Gaussian kernel (\dgs{}, the direct counterpart of N-DGS) and show that the same conditioning extends to the Beta kernel via a Beta-kernel instantiation (\dbs{}, the direct counterpart of UBS). Pairing each direct method with its covariance-derived counterpart isolates the gain from direct conditioning from the choice of kernel.

\item \textbf{Efficient CUDA implementation.} Fused, register-level CUDA kernels for direct conditioning and the N-DGS/UBS baselines achieve 6.9--7.7$\times$ faster production slicing than on-the-fly covariance-derived conditioning. Inference caching controls and complete-render profiling quantify the resulting speedups. The N-DGS baseline was previously only available as a PyTorch implementation; we provide an optimized CUDA comparison.
\end{itemize}

\section{Related Work}
\label{sec:related_work}

\textbf{3D Gaussian Splatting.}
3D Gaussian Splatting (3DGS)~\citep{3dgs} achieves real-time novel view synthesis through explicit 3D Gaussian primitives with spherical harmonic color encoding, and has inspired a wide range of extensions to anti-aliasing~\citep{yu2023mipsplattingaliasfree3dgaussian,yan2024multiscale3dgaussiansplatting}, densification~\citep{kheradmand20243dgaussiansplattingmarkov}, compression~\citep{lee2024compact,morgenstern2024compact}, and semantic modeling~\citep{shi2023languageembedded3dgaussians,zhou2024feature3dgssupercharging3d}. The original formulation remains static, relying on spherical harmonics alone for view dependence and offering no mechanism for temporal dynamics or view-dependent geometry.

\textbf{Alternative Kernel Designs.}
Several works replace the 3D Gaussian kernel with alternatives for better geometric fidelity: generalized exponentials~\citep{hamdi2024gesgeneralizedexponentialsplatting}, truncated~\citep{liu2025tnt}, second-order~\citep{zhang2025quadratic}, half-Gaussian~\citep{li20253d}, and discontinuity-aware~\citep{qu2024disc} variants; Beta~\citep{liu2025deformable}, Student-t~\citep{zhu20253dstudentsplattingscooping}, and Gabor~\citep{zhou20253dgabsplat3dgaborsplatting} distributions; and alternative primitives such as convex shapes~\citep{held20243dconvexsplattingradiance}, triangles~\citep{held2025triangle}, and radiance meshes~\citep{radiancemesh}. Orthogonal to kernel shape, Textured Gaussians~\citep{chao2025textured} models spatially varying color and opacity, while Spherical Voronoi~\citep{di2025spherical} provides a directional appearance representation. These works modify primitives or appearance models rather than replacing covariance-derived conditional slicing; UBS combines the Beta kernel with N-dimensional conditioning, discussed next.

\textbf{N-Dimensional Splatting.}
To model view-dependent and temporal effects, recent methods lift primitives into higher-dimensional spaces, conditioning on view direction or time alongside position. 4DGS~\citep{yang2023real} adds a temporal dimension for dynamic scene reconstruction, and N-DG~\citep{diolatzis2024n} establishes a general framework for arbitrary-dimensional Gaussians. 6DGS~\citep{gao20246dgs} jointly models spatial position and view direction; 7DGS~\citep{gao20257dgs} adds time as well; we refer to the two collectively as N-DGS. Universal Beta Splatting (UBS)~\citep{liu2025universal} extends the same paradigm to Beta kernels with per-dimension bandwidth control. N-DGS and UBS together cover the two kernel families currently used in N-dimensional splatting, and both inherit the covariance-derived conditioning step that our work replaces.

\textbf{Dynamic Scene Modeling.}
NeRF-based methods model dynamics through deformations: D-NeRF~\citep{pumarola2020dnerfneuralradiancefields} learns a canonical scene and a time-dependent deformation field, Nerfies~\citep{park2021nerfies} handles non-rigid motion, and HyperNeRF~\citep{park2021hypernerf} captures topological changes. For Gaussian Splatting, 4DGS~\citep{yang2023real} slices native spatiotemporal primitives for real-time rendering. Dynamic 3D Gaussians~\citep{luiten2023dynamic} optimizes persistent primitives' motion and rotation under local-rigidity constraints, while FreeTimeGS~\citep{wang2025freetimegs} allows Gaussians to appear at arbitrary times and locations with motion functions. Tracked and motion-based representations evolve 3D primitives explicitly; N-dimensional splatting instead obtains query-dependent primitives by conditional slicing.

\section{Preliminary}

\paragraph{3D Gaussian Splatting}
3DGS~\citep{3dgs} represents a scene as a collection of anisotropic 3D Gaussians. Each primitive carries a mean $\boldsymbol{\mu}\in\mathbb{R}^3$, a covariance $\boldsymbol{\Sigma}\in\mathbb{R}^{3\times 3}$ factored as $\boldsymbol{\Sigma}=\mathbf{R}\mathbf{S}\mathbf{S}^\top\mathbf{R}^\top$ with rotation $\mathbf{R}$ (quaternion) and diagonal scale $\mathbf{S}=\operatorname{diag}(s_x,s_y,s_z)$, an opacity $\alpha\in[0,1]$, and view-dependent color via spherical harmonics. At render time, primitives are projected to 2D screen-space Gaussians, sorted by depth per tile, and alpha-composited front-to-back to produce the final pixel color. The per-pixel contribution of each primitive is $\alpha\cdot\mathcal{K}_G(\boldsymbol{\delta};\boldsymbol{\Sigma}_{\text{2D}}) = \alpha\cdot\exp\!\bigl(-\tfrac{1}{2}\boldsymbol{\delta}^\top\boldsymbol{\Sigma}_{\text{2D}}^{-1}\boldsymbol{\delta}\bigr)$, where $\boldsymbol{\delta}$ is the offset from the pixel to the projected primitive center, $\boldsymbol{\Sigma}_{\text{2D}}$ is the projected 2D covariance, and $\mathcal{K}_G$ denotes the Gaussian kernel. The representation is static: it has no mechanism for temporal dynamics or view-dependent geometry.

\paragraph{N-Dimensional Splatting}
To capture view-dependent and temporal effects, N-DGS~\citep{gao20246dgs,gao20257dgs} (Gaussian kernel) and UBS~\citep{liu2025universal} (Beta kernel) parameterize each primitive as a joint distribution over position $\mathbf{X}_p\in\mathbb{R}^3$ and query variables $\mathbf{X}_q\in\mathbb{R}^C$ (view direction for $C=3$, view+time for $C=4$):
\begin{align}
\begin{pmatrix} \mathbf{X}_p \\ \mathbf{X}_q \end{pmatrix} \sim \left(
\begin{pmatrix}
\boldsymbol{\mu}_p \\ \boldsymbol{\mu}_q
\end{pmatrix},
\begin{pmatrix}
\boldsymbol{\Sigma}_{pp} & \boldsymbol{\Sigma}_{pq} \\
\boldsymbol{\Sigma}_{pq}^\top & \boldsymbol{\Sigma}_{qq}
\end{pmatrix}
\right).
\end{align}
At query $\mathbf{q}$, N-DGS derives conditional position, spatial covariance, and opacity as
\begin{align}
\boldsymbol{\mu}_{\text{cond}} &= \boldsymbol{\mu}_p + \boldsymbol{\Sigma}_{pq}\,\boldsymbol{\Sigma}_{qq}^{-1}(\mathbf{q}-\boldsymbol{\mu}_q), \label{eq:prelim_mu}\\
\boldsymbol{\Sigma}_{\text{cond}} &= \boldsymbol{\Sigma}_{pp} - \boldsymbol{\Sigma}_{pq}\,\boldsymbol{\Sigma}_{qq}^{-1}\,\boldsymbol{\Sigma}_{pq}^\top, \label{eq:prelim_sigma}\\
\alpha_{\text{cond}} &= \alpha \cdot \mathcal{K}\!\left(\mathbf{q}-\boldsymbol{\mu}_q;\,\boldsymbol{\Sigma}_{qq}\right). \label{eq:prelim_alpha}
\end{align}
$\mathcal{K}(\cdot;\boldsymbol{\Sigma}_{qq})$ is evaluated through the query precision $\boldsymbol{\Sigma}_{qq}^{-1}$. N-DGS uses $\mathcal{K}_{G}(\boldsymbol{\delta};\boldsymbol{\Sigma}_{qq}) = \exp(-\lambda\,\boldsymbol{\delta}^\top \boldsymbol{\Sigma}_{qq}^{-1}\boldsymbol{\delta})$ with scaling $\lambda$ (default 0.35). UBS retains the covariance-derived structure with a Beta opacity kernel and bandwidth-dependent mean/covariance coupling (Appendix~\ref{sec:dbs}). Both require query-block inversion, regression construction, and covariance correction; our direct formulation exposes the rendered quantities instead.

\begin{figure*}[t]
\centering

\includegraphics[width=0.99\linewidth]{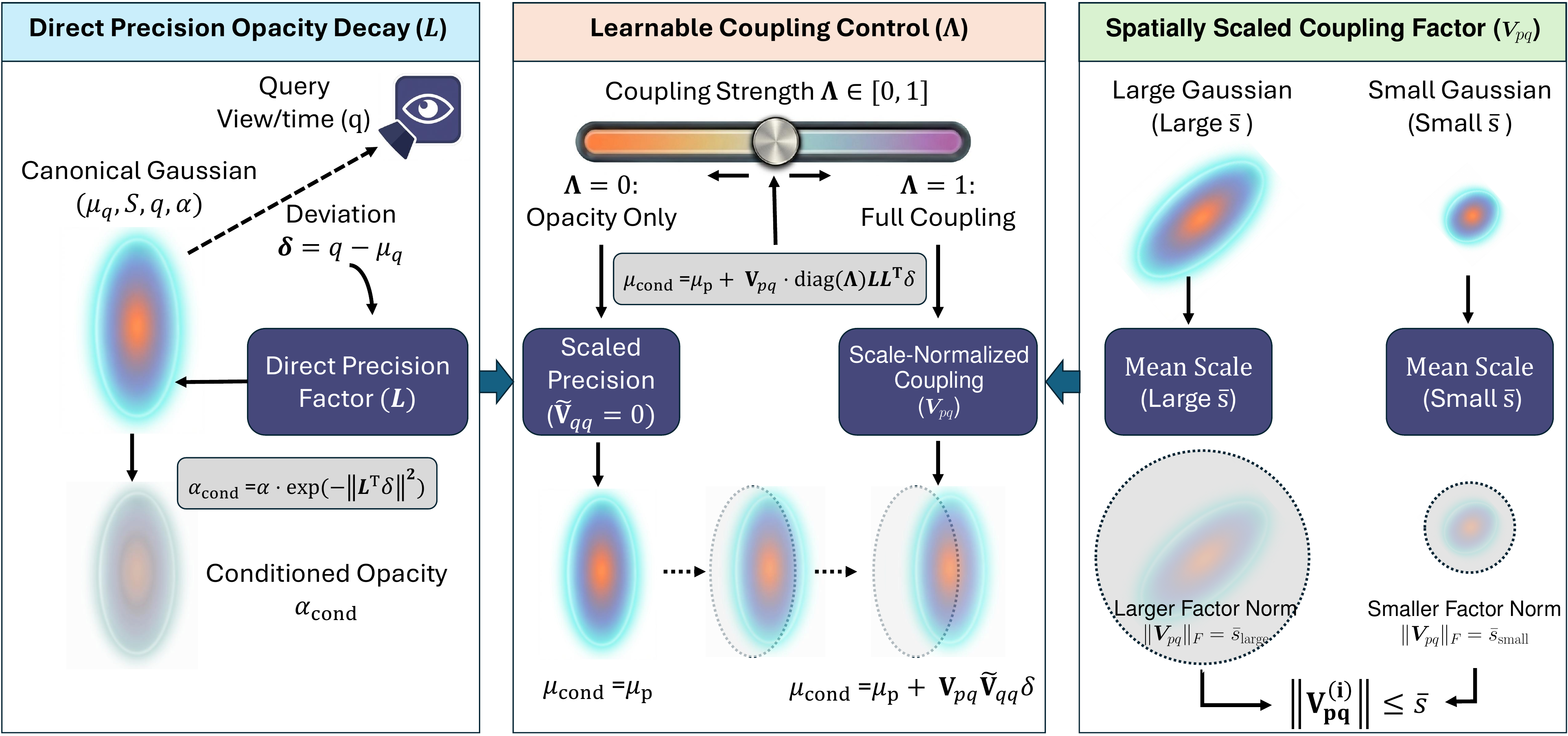}
\caption{Overview of the \dgs{} pipeline. Given query deviation $\boldsymbol{\delta} = \mathbf{q} - \boldsymbol{\mu}_q$, opacity uses Cholesky factor $\mathbf{L}$, while position uses the scale-normalized coupling factor $\mathbf{V}_{pq}$ and scaled precision $\tilde{\mathbf{V}}_{qq}$ with per-dimension coupling $\boldsymbol{\Lambda}$. The factor norm follows the primitive scale; the complete regression and final shift also depend on $\tilde{\mathbf{V}}_{qq}$ and $\boldsymbol{\delta}$. The same pipeline extends to \dbs{} (\cref{subsec:dbs}).}
\label{fig:architecture}
\vspace{-1em}
\end{figure*}

\section{Our Approach}

We replace the covariance-derived conditioning step with three explicit parameters per primitive: a Cholesky factor $\mathbf{L}$ for opacity precision, a spatial coupling factor $\mathbf{V}_{pq}$ for position, and a per-dimension coupling vector $\boldsymbol{\Lambda}$ that ties the two. We develop the formulation for the Gaussian kernel below, yielding \dgs{}, and show in \cref{subsec:dbs} that swapping in the Beta opacity decay yields \dbs{}.

\subsection{Overview}
\label{subsec:overview}

Given the deviation $\boldsymbol{\delta} = \mathbf{q} - \boldsymbol{\mu}_q$ from the canonical view/time, \dgs{} replaces \cref{eq:prelim_mu,eq:prelim_alpha} with two directly parameterized components:
\begin{align}
\alpha_{\text{cond}} &= \alpha \cdot \exp\!\left(-\lambda_o\,\|\mathbf{L}^\top\boldsymbol{\delta}\|_2^2\right), & &\text{(Opacity)} \label{eq:direct_alpha}\\
\boldsymbol{\mu}_{\text{cond}} &= \boldsymbol{\mu}_p + \mathbf{V}_{pq}\,\mathrm{diag}(\boldsymbol{\Lambda})\,\mathbf{V}_{qq}\,\boldsymbol{\delta}. & &\text{(Position)} \label{eq:direct_mu}
\end{align}
The parameterization is defined by three learned quantities ($\mathbf{L}$, $\mathbf{V}_{pq}$, $\boldsymbol{\Lambda}$) and one fixed constant ($\lambda_o{=}0.35$, the opacity scaling factor inherited from N-DGS). $\mathbf{L}\in\mathbb{R}^{C\times C}$ is a lower-triangular Cholesky factor with an exponential activation on the diagonal, so that $\mathbf{V}_{qq}=\mathbf{L}\mathbf{L}^\top$ is symmetric positive-definite by construction and \emph{is} the precision matrix without any inversion. $\mathbf{V}_{pq}=\bar{s}\boldsymbol{\Theta}_{pq}\in\mathbb{R}^{3\times C}$ is a spatial coupling factor, written as the average spatial scale $\bar{s}=\mathrm{mean}(s_x,s_y,s_z)$ times a learnable direction $\boldsymbol{\Theta}_{pq}$ with unit Frobenius norm. This gives $\|\mathbf{V}_{pq}\|_F=\bar{s}$; the effective regression also depends on $\operatorname{diag}(\boldsymbol{\Lambda})\mathbf{V}_{qq}$. $\boldsymbol{\Lambda}\in[0,1]^C$ weights the precision-transformed query channels in the position update: at $\boldsymbol{\Lambda}=\mathbf{0}$ position decouples entirely (recovering an opacity-only variant we denote \dgso{}, with $\boldsymbol{\mu}_{\text{cond}}=\boldsymbol{\mu}_p$), at $\boldsymbol{\Lambda}=\mathbf{1}$ it couples through the full opacity precision $\mathbf{V}_{qq}$, and intermediate values continuously weight these channels. Channels of the same type (\eg, the three view dimensions) share a logit.

The rendered conditional covariance $\boldsymbol{\Sigma}_{\mathrm{cond}}=\mathbf{R}\mathbf{S}\mathbf{S}^\top\mathbf{R}^\top$ remains query-independent and unchanged from 3DGS, preserving compatibility with existing densification strategies. It differs from the spatial marginal of a reconstructed joint Gaussian: Appendix~\ref{sec:conditional_coordinates} gives the exact correspondence for unrestricted conditional coordinates. The implemented factorization constrains the regression through spatial normalization and bounded coupling, changing its optimization geometry and priors.

\begin{table}[t!]
\centering
\caption{Per-primitive operations and production slicing-only forward time on A100 (1M primitives). Direct methods avoid query-block inversion, regression construction, and covariance correction. Parentheses compare optimized on-the-fly conditioning; Appendix~\ref{app:efficiency-controls} reports separate cached and factorized inference controls.}
\label{tab:comparison_overview}
\setlength{\tabcolsep}{3pt}
\resizebox{0.85\textwidth}{!}{%
\begin{tabular}{@{}l|cc|cc@{}}
\toprule
& \multicolumn{2}{c|}{\textbf{Implicit}} & \multicolumn{2}{c}{\textbf{Direct (ours)}} \\
\textbf{Operation} & \textbf{N-DGS} & \textbf{UBS} & \textbf{\dgs{}} & \textbf{\dbs{}} \\
\midrule
Invert $\boldsymbol{\Sigma}_{qq}$ & $\boldsymbol{\Sigma}_{qq}^{-1}$ & $\boldsymbol{\Sigma}_{qq}^{-1}$ & --- & --- \\
Regr.\ matrix $\mathbf{P}$ & $\boldsymbol{\Sigma}_{pq}\boldsymbol{\Sigma}_{qq}^{-1}$ & $\boldsymbol{\Sigma}_{pq}\boldsymbol{\Sigma}_{qq}^{-1}\!\mathrm{diag}(\boldsymbol{\Lambda}_\beta)$ & --- & --- \\
Cov.\ correction & $\mathbf{P}\,\boldsymbol{\Sigma}_{pq}^\top$ & $\mathbf{P}\,\boldsymbol{\Sigma}_{pq}^\top$ & --- & --- \\
Position shift & $\mathbf{P}\,\boldsymbol{\delta}$ & $\mathbf{P}\,\boldsymbol{\delta}$ & $\mathbf{V}_{pq}\mathrm{diag}(\boldsymbol{\Lambda})\mathbf{V}_{qq}\boldsymbol{\delta}$ & $\mathbf{V}_{pq}\mathrm{diag}(\boldsymbol{\Lambda}_\beta)\mathbf{V}_{qq}\boldsymbol{\delta}$ \\
Opacity & $\mathcal{K}_G(\boldsymbol{\delta}^\top\!\boldsymbol{\Sigma}_{qq}^{-1}\!\boldsymbol{\delta})$ & $\mathcal{K}_B(\hat{\mathbf{L}}_q^{-1}\!\boldsymbol{\delta};\boldsymbol{\beta})$ & $\mathcal{K}_G(\|\mathbf{L}^\top\!\boldsymbol{\delta}\|^2)$ & $\mathcal{K}_B(\mathbf{L}^\top\!\boldsymbol{\delta};\boldsymbol{\beta})$ \\
\midrule
Slicing time, $C{=}3$ (ms) & 0.71 & 0.67 & \textbf{0.10} ($\bm{7.1\times}$) & \textbf{0.10} ($\bm{6.9\times}$) \\
Slicing time, $C{=}4$ (ms) & 0.96 & 1.34 & \textbf{0.12} ($\bm{7.7\times}$) & \textbf{0.13} ($\bm{7.2\times}$) \\
\bottomrule
\end{tabular}%
}
\vspace{-1em}
\end{table}

Table~\ref{tab:comparison_overview} contrasts the per-primitive operations: N-DGS constructs $\mathbf{P}=\boldsymbol{\Sigma}_{pq}\boldsymbol{\Sigma}_{qq}^{-1}$ through query-block inversion and computes a covariance correction; UBS additionally applies bandwidth-dependent coupling. Direct methods apply their regression operator without deriving these quantities. The coupling $\boldsymbol{\Lambda}$ adds explicit control: $\boldsymbol{\Lambda}=\mathbf{1}$ uses the full direct precision path and $\boldsymbol{\Lambda}=\mathbf{0}$ recovers \dgso{}. Intermediate values let each query type learn a partial coupling strength.

\subsection{Opacity Conditioning}

The Cholesky factor $\mathbf{L}$ is parameterized directly, with exponential activation on the diagonal so that $\mathbf{V}_{qq}=\mathbf{L}\mathbf{L}^\top$ is symmetric positive-definite. Evaluating opacity reduces to a triangular matrix-vector product $\mathbf{L}^\top\boldsymbol{\delta}$ followed by a squared norm, avoiding query-block inversion.

For $C=4$ (view+time), all ten lower-triangular entries of $\mathbf{L}$ are learned. View--time off-diagonals start at zero but remain unconstrained. With $\boldsymbol{\delta}=(\boldsymbol{\delta}_v,\delta_t)^\top$, opacity is
\begin{equation}
\alpha_{\text{cond}} = \alpha \cdot \exp\!\left(-\lambda_o\left(\boldsymbol{\delta}_v^\top \mathbf{V}_{vv} \boldsymbol{\delta}_v + 2\delta_t\boldsymbol{\delta}_v^\top\mathbf{v}_{vt} + \delta_t^2 v_t\right)\right),
\end{equation}
where $\mathbf{V}_{vv}$, $\mathbf{v}_{vt}$, and $v_t$ are the view, cross, and time blocks of the full precision $\mathbf{V}_{qq}$.

\subsection{Position Conditioning}
\label{subsec:position}

The position update~\eqref{eq:direct_mu} factors into the spatial coupling factor $\mathbf{V}_{pq}$, the per-dimension coupling $\mathrm{diag}(\boldsymbol{\Lambda})$, and the precision $\mathbf{V}_{qq}$ shared with opacity:
\begin{equation*}
\boldsymbol{\mu}_{\text{cond}} = \boldsymbol{\mu}_p + \underbrace{\bar{s}\,\boldsymbol{\Theta}_{pq}}_{\mathbf{V}_{pq}}\,\underbrace{\mathrm{diag}(\boldsymbol{\Lambda})\,\mathbf{V}_{qq}}_{\tilde{\mathbf{V}}_{qq}}\,\boldsymbol{\delta}.
\end{equation*}
The coupling $\boldsymbol{\Lambda}$ interpolates between two extremes per primitive: at $\boldsymbol{\Lambda}=\mathbf{0}$, position is query-independent (recovering \dgso{}); at $\boldsymbol{\Lambda}=\mathbf{1}$, position and opacity share the full precision $\mathbf{V}_{qq}$, mirroring N-DGS's coupling structure. Precision-transformed query channels can learn different coupling strengths.

With $\boldsymbol{\Lambda}=\mathbf{0}$, opacity conditioning remains active
under \cref{eq:direct_alpha}. A primitive can therefore be suppressed under an
unfavorable view without moving its center; stronger coupling permits
query-dependent motion with the same base spatial covariance.

\textbf{Spatial scaling.} We set $\mathbf{V}_{pq}=\bar{s}\,\boldsymbol{\Theta}_{pq}$ with $\|\boldsymbol{\Theta}_{pq}\|_F=1$, giving $\|\mathbf{V}_{pq}\|_F=\bar{s}$ and $\|\mathbf{V}_{pq}^{(i)}\|_2\le\bar{s}$. This is an explicit spatial-scale prior on the coupling factor, not a bound on the final shift: the effective regression is $\mathbf{M}_{\mathrm{dGS}}=\mathbf{V}_{pq}\operatorname{diag}(\boldsymbol{\Lambda})\mathbf{V}_{qq}$ and the shift also depends on the query deviation. N-DGS instead obtains its regression operator indirectly from shared covariance blocks. Appendix~\ref{sec:spatial_scaling_proof} derives both forms and a basis-independent Gaussian bound; the ablation in \cref{tab:ablation} confirms the value of the prior, as removing $\bar{s}$ degrades PSNR by 2.65 dB on NeRF Synthetic.
The mean scale follows primitive size symmetrically across principal axes, while query precision and $\boldsymbol{\Lambda}$ retain separate control of opacity and position coupling.

\subsection{Extension to Beta Kernel: \dbs{}}
\label{subsec:dbs}

Direct conditioning touches only the slicing step (precision read-out, regression, and coupling), making it kernel-agnostic: any opacity decay that consumes the precision-Cholesky distance $\mathbf{L}^\top\boldsymbol{\delta}$ can be substituted in. We instantiate this with a Beta opacity decay following UBS~\citep{liu2025universal} to obtain \dbs{} (direct Beta Splatting), the direct counterpart of UBS.

We follow the released UBS implementation used in our experiments. UBS evaluates opacity as $\alpha\prod_i(1-\tanh((\hat{\mathbf{L}}_q^{-1}\boldsymbol{\delta})_i^2))^{\beta_i}$, with $\hat{\mathbf{L}}_q\hat{\mathbf{L}}_q^\top=\boldsymbol{\Sigma}_{qq}$. \dbs{} replaces this triangular solve with the direct product $\mathbf{z}=\mathbf{L}^\top\boldsymbol{\delta}$ while retaining the same bandwidth activation and coupling convention:
\begin{align}
\alpha_{\text{cond}} &= \alpha \cdot \prod_{i=1}^{C} \bigl(1-\tanh(z_i^2)\bigr)^{\beta_i}, \label{eq:dbs_alpha}\\
\boldsymbol{\mu}_{\text{cond}} &= \boldsymbol{\mu}_p + \mathbf{V}_{pq}\,\mathrm{diag}(\boldsymbol{\Lambda}_\beta)\,\mathbf{V}_{qq}\,\boldsymbol{\delta}, \label{eq:dbs_mu}
\end{align}
where $\beta_i=4\exp(\beta_i^{\mathrm{raw}})$ and $\boldsymbol{\Lambda}_\beta=\min(\boldsymbol{\beta}/4,\mathbf{1})$, as in UBS. The position update reuses the opacity bandwidth for coupling instead of a separate learned $\boldsymbol{\Lambda}$. The precision factor, spatial factor, and directly stored conditional covariance transfer from \dgs{}. Including the spatial Beta parameter, UBS and \dbs{} each use 19/27 conditioning parameters for $C=3/4$; N-DGS and \dgs{} use 15/16 and 22/24, respectively. Appendix~\ref{sec:dbs} gives the source-aligned UBS formulation.

The shared bandwidth gives opacity and position a consistent response to the
query. Dividing $\boldsymbol{\beta}$ by four puts the coupling near unit
magnitude when its raw parameter is near zero, while clipping prevents a
larger bandwidth from increasing the coupling beyond one. Gaussian and Beta
instances thus retain the same direct conditional structure, with a separate
learned coupling in \dgs{} and a bandwidth-derived coupling in \dbs{}.

\subsection{CUDA Implementation}

The original N-DGS methods (6DGS~\citep{gao20246dgs}, 7DGS~\citep{gao20257dgs}) rely on PyTorch's autograd for the conditional Gaussian operations, so for a fair comparison we implement \dgs{} and the N-DGS baseline with custom CUDA kernels. The kernel fuses the conditional position and opacity computation in a single pass (Algorithm~\ref{alg:slice} in Appendix~\ref{sec:algorithms}), and \dbs{} reuses the same pipeline with the opacity decay swapped (Algorithm~\ref{alg:slice_dbs}).

\textbf{Kernel design.} The kernel is templated on $C$ ($3$ or $4$) with compile-time loop unrolling, and each thread processes one primitive. The triangular product $\mathbf{L}^\top\boldsymbol{\delta}$ uses only $C(C+1)/2$ multiplications followed by a dot product, all in per-thread registers; the small $3\times C$ and $C\times C$ matrices fit entirely in registers, so the kernel allocates no intermediate tensors and avoids shared-memory traffic.
The N-DGS baseline also fuses inversion, regression construction, and Schur
correction, using a cofactor inverse for $C=3$ and register-local
Gauss--Jordan elimination for $C=4$.

\begin{figure*}[!t]
    \centering
\includegraphics[width=1.0\linewidth]{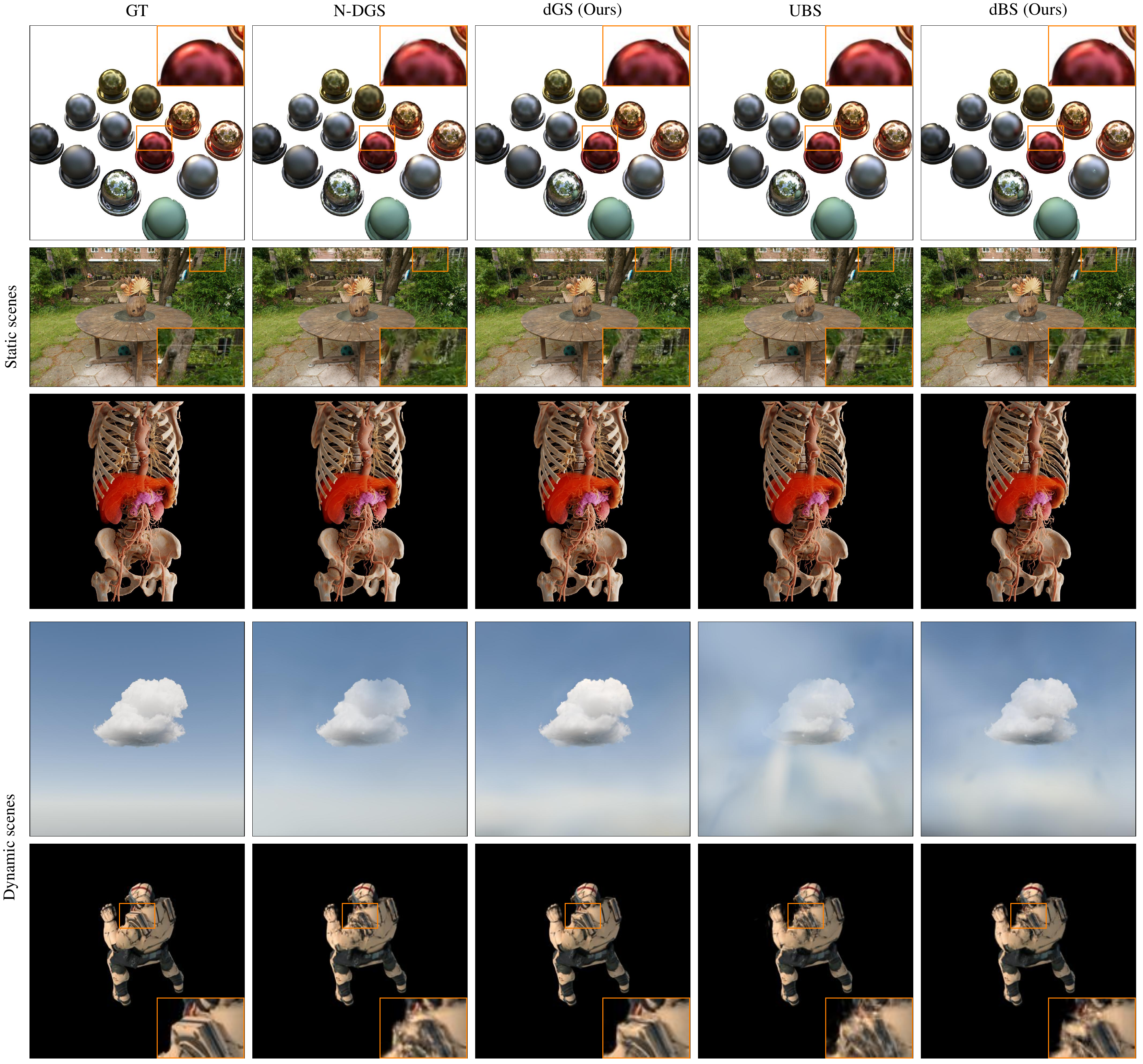}
    \caption{Qualitative comparison on static and dynamic scenes. Insets use identical crop coordinates and magnification across methods; boxes mark the source regions.}
    \label{fig:qualitative}
    \vspace{-1em}
\end{figure*}

\textbf{Kernel benchmarks.} Table~\ref{tab:comparison_overview} reports production CUDA slicing at 1M primitives on an A100, excluding autograd and rasterization. Direct conditioning is \textbf{6.9--7.7$\times$ faster} than optimized on-the-fly conditioning. In a separate harness with materialized inputs, \dgs{} remains 2.0--2.3$\times$ faster than inference-cached N-DGS and 1.8$\times$ faster than a factorized-query control (Appendix~\ref{app:efficiency-controls}). Profiling four trained scenes gives 1.25--1.91$\times$ complete-render gains: shared rasterizer code still processes different learned footprints and tile workloads, limiting the benefit of faster slicing.

\begin{table*}[t]
\centering
\caption{MCMC results: scene means from one run per scene (\#Gauss in K, training time in minutes; dynamic batch size 1). Pairs share per-scene primitive caps; realized counts are reported. FPS and FPS$_{\text{CUDA}}$ are arithmetic scene means using PyTorch and fused CUDA slicing.}
\label{tab:mcmc_comparison}
\resizebox{0.9\textwidth}{!}{%
\begin{tabular}{@{}ccl|rrrr|ccc@{}}
\toprule
& Dataset & Method & \#Gauss & Time $\downarrow$ & FPS $\uparrow$& FPS$_{\text{CUDA}}$ $\uparrow$& PSNR $\uparrow$ & SSIM $\uparrow$ & LPIPS $\downarrow$ \\
\midrule
\multirow{12}{*}{\rotatebox{90}{STATIC}}
& \multirow{4}{*}{\makecell{\texttt{NeRF}\\\texttt{Synthetic}}}
& N-DGS & 300 & 10.43 & 183.58 & 348.95 & 34.04 & 0.973 & 0.027 \\
&& \textit{d}GS (\textbf{Ours}) & 300 & \bf 10.26 & \bf 245.18 & \bf 593.94 &  \bf 34.34 &  \bf 0.973 &  \bf 0.026 \\
\spacedcdashline{3-10}
&& UBS & 300 & 6.80 & 155.96 & 306.33 & 34.48 & 0.975 & 0.026 \\
&& \textit{d}BS (\textbf{Ours}) & 300 &  \bf5.48 & \bf 239.09 &  \bf 415.66 & \bf 34.53 & \bf 0.975 & \bf 0.025 \\
\cmidrule{2-10}
& \multirow{4}{*}{\makecell{\texttt{Mip-NeRF}\\\texttt{360}}}
& N-DGS & 2,852 & 51.23 & 35.69 & 75.31 & 26.88 & 0.775 & 0.263 \\
&& \textit{d}GS (\textbf{Ours}) & 2,574 & \bf 49.85 & \bf 42.49 & \bf 200.56 &  \bf 27.67 & \bf 0.794 & \bf  0.227 \\
\spacedcdashline{3-10}
&& UBS & 3,057 & 28.59 & 40.94 & 76.18 & 28.50 & 0.841 & \bf 0.183 \\
&& \textit{d}BS (\textbf{Ours}) & 3,057 &  \bf 24.42 & \bf 62.05 &  \bf 157.45 & \bf 28.70 & \bf 0.842 & 0.185 \\
\cmidrule{2-10}
& \multirow{4}{*}{\texttt{6DGS-PBR}}
& N-DGS & 163 & 54.80 & 328.75 & 476.79 & 40.06 & 0.976 & 0.039 \\
&& \textit{d}GS (\textbf{Ours}) & 163 & \bf 47.54 & \bf 356.15 & \bf 551.02 & \bf 40.34 & \bf 0.976 & \bf 0.038 \\
\spacedcdashline{3-10}
&& UBS & 163 & 8.70 & 239.87 & 348.63 & 39.53 & 0.975 & 0.042 \\
&& \textit{d}BS (\textbf{Ours}) & 163 & \bf 8.47 & \bf 305.33 & \bf 401.93 & \bf 39.72 & \bf 0.976 & \bf 0.042 \\
\midrule
\midrule
\multirow{8}{*}{\rotatebox{90}{DYNAMIC}}
& \multirow{4}{*}{\texttt{D-NeRF}}
& N-DGS & 150 & \bf 10.31 & 306.51 & 501.52 & 34.63 & 0.976 & 0.033 \\
&& \textit{d}GS (\textbf{Ours}) & 150 &  10.61 & \bf 407.40 & \bf 746.65 & \bf 34.64 & \bf 0.976 & \bf 0.030 \\
\spacedcdashline{3-10}
&& UBS & 150 & \bf 8.37 & 229.83 & 371.81 & 32.91 & 0.967 & 0.040 \\
&& \textit{d}BS (\textbf{Ours}) & 150 & 9.34 & \bf 323.75 & \bf 463.09 & \bf 33.55 & \bf 0.972 & \bf 0.033 \\
\cmidrule{2-10}
& \multirow{4}{*}{\texttt{7DGS-PBR}}
& N-DGS & 175 & \bf 28.44 & 294.48 & 511.12 & 29.20 & 0.944 & 0.063 \\
&& \textit{d}GS (\textbf{Ours}) & 175 & 31.12 & \bf 366.06 & \bf 577.38 & \bf 29.90 & \bf 0.948 & \bf 0.061 \\
\spacedcdashline{3-10}
&& UBS & 175 & 9.90 & 232.17 & 396.33 & 30.61 & 0.951 & 0.057 \\
&& \textit{d}BS (\textbf{Ours}) & 175 & \bf 8.79 & \bf 347.93 & \bf 539.57 & \bf 31.87 & \bf 0.955 & \bf 0.053 \\
\bottomrule
\end{tabular}}
\vspace{-2em}
\end{table*}

\section{Experiments}
\label{sec:experiments}

\subsection{Experimental Setup}

\textbf{Datasets.}
We evaluate the \dfamily{} family across five diverse datasets covering both static and dynamic scenes:
\vspace{-1em}
\begin{itemize}[leftmargin=*]
    \item \textit{Static scenes:} (1) \textit{NeRF Synthetic}~\citep{mildenhall2021nerf} with 8 synthetic scenes under controlled lighting for evaluating geometric reconstruction; (2) \textit{Mip-NeRF 360}~\citep{barron2022mipnerf360} with 9 challenging unbounded real-world scenes for evaluating view synthesis; and (3) \textit{6DGS-PBR}~\citep{gao20246dgs} featuring 7 physically-based scenes with complex view-dependent effects.

    \item \textit{Dynamic scenes:} (4) \textit{D-NeRF}~\citep{pumarola2020dnerfneuralradiancefields} with 8 synthetic sequences featuring various deformation types; and (5) \textit{7DGS-PBR}~\citep{gao20257dgs} containing 6 physically-based dynamic scenes with complex spatio-temporal-angular correlations.
\end{itemize}

\textbf{Evaluation Metrics.}
We assess reconstruction quality using Peak Signal-to-Noise Ratio (PSNR), Structural Similarity Index (SSIM)~\citep{wang2004image}, and Learned Perceptual Image Patch Similarity (LPIPS)~\citep{zhang2018perceptual}. We follow the 3DGS~\citep{3dgs} evaluation protocol (joint-RGB PSNR, white background on NeRF Synthetic). We re-evaluate UBS~\citep{liu2025universal} under this protocol; its released evaluator instead averages channel-wise PSNR and uses a black background on NeRF Synthetic. We also report training time (minutes), Gaussian count (K), and two rendering speeds: FPS uses a PyTorch reference implementation for slicing, while FPS$_{\text{CUDA}}$ uses the fused CUDA slicing kernels.

\textbf{Implementation Details.}
All methods are implemented in PyTorch with custom CUDA kernels on a single NVIDIA A100 GPU and rendered with the TC-GS rasterizer~\citep{liao2025tc}. We compare two head-to-head pairs (\dgs{} vs.\ N-DGS, \dbs{} vs.\ UBS) and additionally include \dgso{} (\dgs{} with $\boldsymbol{\Lambda}{=}\mathbf{0}$) as a maximum-efficiency variant. All methods share the same training pipeline: Adam optimizer, 30K iterations, identical learning rates and densification schedules per dataset. Main dynamic MCMC results use batch size 1 for both methods; Appendix~\ref{app:robustness} checks the Beta pair at the original batch size 4. Main benchmark entries are single-run scene averages; repeated-seed controls are reported separately. For \dgs{} and \dbs{}, the view mean $\boldsymbol{\mu}_q$ is initialized to random unit vectors (with time uniformly sampled in $[0,1]$ for dynamic scenes), the coupling-direction parameters $\boldsymbol{\Theta}_{pq}$ from $\mathcal{N}(0,0.01^2)$, and the Cholesky factor $\mathbf{L}$ with diagonal elements $\log(2.0)$ and zero off-diagonals.

\subsection{Results with MCMC Densification}

Table~\ref{tab:mcmc_comparison} compares each direct method against its covariance-derived counterpart under MCMC densification~\citep{kheradmand20243dgaussiansplattingmarkov}. Pairs share per-scene primitive caps, while realized counts can differ. We cap primitives at 300K for NeRF Synthetic and 150K for D-NeRF, with scene-specific caps for Mip-NeRF 360 (1.5M--6M), 6DGS-PBR (60K--300K), and 7DGS-PBR (150K--300K).

\textbf{Gaussian-kernel pair (\dgs{} vs.\ N-DGS).} \dgs{} matches or improves dataset-mean quality while increasing rendering throughput. On Mip-NeRF 360, it gains +0.79 dB with 2.574M versus 2.852M primitives, approximately 10\% fewer. Arithmetic-mean FPS$_{\text{CUDA}}$ is 200.56 versus 75.31 (a 2.66$\times$ ratio); this aggregation differs from mean frame time, and realized counts also affect speed. NeRF Synthetic gains +0.30 dB and +70\% mean FPS (593.94 vs.\ 348.95). D-NeRF gives quality parity (+0.01 dB) and +49\% mean FPS (746.65 vs.\ 501.52). On 7DGS-PBR, which combines dynamics with view-dependent specular effects, \dgs{} gains +0.70 dB and +13\% mean FPS (577.38 vs.\ 511.12).

\textbf{Beta-kernel pair (\dbs{} vs.\ UBS).} The same throughput pattern transfers to the Beta kernel. On NeRF Synthetic, \dbs{} gives quality parity (+0.05 dB) while rendering +36\% faster (415.66 vs.\ 306.33) and training in 5.48 vs.\ 6.80 minutes. On Mip-NeRF 360, \dbs{} improves PSNR by +0.20 dB and renders +107\% faster (157.45 vs.\ 76.18). On D-NeRF, \dbs{} gains +0.64 dB with +25\% faster rendering (463.09 vs.\ 371.81). The 7DGS-PBR gap is largest: +1.26 dB and +36\% (539.57 vs.\ 396.33), with per-scene PSNR gains reaching +2.06 dB on \texttt{heart}. These results support computational gains across kernels; quality depends on the full parameterization and its optimization priors.

\textbf{Slicing and complete-render cost.} The two FPS columns compare PyTorch and fused CUDA slicing within each method. On Mip-NeRF 360, this increases N-DGS mean FPS by 2.11$\times$ and \dgs{} mean FPS by 4.72$\times$. On four profiled checkpoints, slicing occupies 42--64\% of baseline frame time; projection, binning, sorting, and rasterization account for the remainder. Appendix~\ref{app:efficiency-controls} separates these costs and explains why shared downstream code does not imply equal rendering workload.

Figure~\ref{fig:qualitative} shows qualitative comparisons. On the three static scenes, \dgs{}, UBS, and \dbs{} all produce visually clean reconstructions, while \dgs{} shows noticeably fewer artifacts than N-DGS. On the two dynamic scenes, \dgs{} and \dbs{} both deliver better visual quality than their covariance-derived counterparts (N-DGS and UBS, respectively). Standard-densification results for \dgs{} vs.\ N-DGS, the protocol N-DGS originally evaluates with, appear in Appendix~\ref{sec:detailed_results}. A broader comparison against SOTA static-scene methods on Mip-NeRF 360 and NeRF Synthetic, where \dbs{} achieves the best PSNR on Mip-NeRF 360 and ties for best PSNR on NeRF Synthetic, is provided in Appendix~\ref{app:sota}.

\begin{table}[!t]
\centering
\caption{Ablation study on NeRF Synthetic and 6DGS-PBR. $\Lambda{=}0$ gives \textit{d}GS-O (opacity only), $\Lambda{=}1$ fixes full coupling, and learned $\Lambda$ gives full \textit{d}GS. ``w/o $\bar{s}$'' removes spatial scaling.}
\label{tab:ablation}
\setlength{\tabcolsep}{4pt}
\resizebox{0.7\textwidth}{!}{%
\begin{tabular}{l|ccc|ccc}
\toprule
 & \multicolumn{3}{c|}{\texttt{NeRF Synthetic}} & \multicolumn{3}{c}{\texttt{6DGS-PBR}} \\
Config & PSNR $\uparrow$ & SSIM $\uparrow$ & LPIPS $\downarrow$ & PSNR $\uparrow$ & SSIM $\uparrow$ & LPIPS $\downarrow$ \\
\midrule
$\Lambda{=}0$ (\textit{d}GS-O) & 33.84 & 0.971 & 0.030 & 36.54 & 0.965 & 0.059 \\
$\Lambda{=}1$ (fixed) & 33.78 & 0.971 & 0.029 & 37.31 & \textbf{0.972} & 0.053 \\
\textit{d}GS w/o $\bar{s}$ & 31.26 & 0.957 & 0.046 & 34.78 & 0.951 & 0.075 \\
$\Lambda$ learned (\textit{d}GS) & \textbf{33.91} & \textbf{0.972} & \textbf{0.029} & \textbf{37.75} & \textbf{0.972} & \textbf{0.052} \\
\bottomrule
\end{tabular}}%
\vspace{-1em}
\end{table}

\subsection{Ablation Study}

Table~\ref{tab:ablation} ablates the two key design choices: the per-dimension coupling $\boldsymbol{\Lambda}$ and the spatial scaling $\bar{s}$ in the coupling factor $\mathbf{V}_{pq}$. All ablation runs use standard 3DGS densification~\citep{3dgs} on NeRF Synthetic and 6DGS-PBR (the same protocol as the appendix Standard-Densification results).
An unrestricted direct-regression control reaches 32.91/36.44 dB versus \dgs{}'s 33.91/37.75 dB, supporting the role of structure and optimization priors beyond direct readout (Appendix~\ref{app:robustness}).

\textbf{Learned $\boldsymbol{\Lambda}$ outperforms both fixed extremes.} On NeRF Synthetic, learned $\boldsymbol{\Lambda}$ achieves 33.91 dB vs.\ 33.84 ($\boldsymbol{\Lambda}{=}\mathbf{0}$, opacity-only) and 33.78 ($\boldsymbol{\Lambda}{=}\mathbf{1}$, full coupling). The gap widens on 6DGS-PBR (view-dependent specular scenes): 37.75 dB vs.\ 36.54 ($\boldsymbol{\Lambda}{=}\mathbf{0}$, $-1.21$ dB) and 37.31 ($\boldsymbol{\Lambda}{=}\mathbf{1}$, $-0.44$ dB). These results suggest that the optimal coupling is scene-dependent and may be primitive-dependent: the diffuse NeRF Synthetic scenes prefer weak coupling ($\boldsymbol{\Lambda}{=}\mathbf{0}$ beats $\boldsymbol{\Lambda}{=}\mathbf{1}$), the specular 6DGS-PBR scenes prefer strong coupling ($\boldsymbol{\Lambda}{=}\mathbf{1}$ beats $\boldsymbol{\Lambda}{=}\mathbf{0}$ by +0.77 dB), and learned per-Gaussian $\boldsymbol{\Lambda}$ improves on both extremes on both datasets.

\textbf{Spatial scaling is critical.} Removing $\bar{s}$ degrades PSNR by 2.65 dB on NeRF Synthetic (33.91 vs.\ 31.26) and 2.97 dB on 6DGS-PBR (37.75 vs.\ 34.78). Notably, an unscaled coupling factor is worse than no position regression at all (the $\boldsymbol{\Lambda}{=}\mathbf{0}$ baseline at 33.84 / 36.54 beats \dgs{} w/o $\bar{s}$ at 31.26 / 34.78). This supports the spatial-scale prior on $\mathbf{V}_{pq}$ analyzed in Appendix~\ref{sec:spatial_scaling_proof}; the effective regression and final shift additionally depend on $\operatorname{diag}(\boldsymbol{\Lambda})\mathbf{V}_{qq}$ and the query deviation.

\textbf{Scope and limitations.}
Our evaluation covers five static and dynamic benchmarks; other kernels and
city-scale scenes require separate validation of memory use and visibility handling.

\section{Conclusion}

We presented direct conditional parameterization for N-dimensional splatting. Whereas N-DGS derives rendered quantities from a joint covariance through inversion and Schur complementation, \dgs{} directly parameterizes conditional quantities with explicit spatial-scale and coupling priors. Unrestricted conditional coordinates preserve the joint Gaussian covariance family, while the structured factorization makes rendering parameters explicit and gives optimization direct control over their coupling. Across five benchmarks and two kernels, the resulting methods match or improve aggregate quality. Production slicing is 6.9--7.7$\times$ faster than on-the-fly conditioning, with smaller complete-render gains determined by downstream workload. Future work can extend the formulation to other primitive attributes, conditioning queries, and kernel families.

\clearpage
\noindent\textbf{Funding and competing interests.}
This work was supported by United Imaging Intelligence through employment
and institutional computing resources. All authors are affiliated with the
company and report no additional third-party funding or competing financial interests.

\bibliographystyle{plainnat}
\bibliography{references}

\newpage
\appendix

\section{Comparison with State-of-the-art Methods}
\label{app:sota}

The main paper compares each direct method against its covariance-derived counterpart in matched pairs to isolate the contribution of conditioning. Here we situate \dbs{} against a broader range of static-scene methods on Mip-NeRF 360~\citep{barron2022mipnerf360} and NeRF Synthetic~\citep{mildenhall2021nerf}, including implicit methods (Instant-NGP, Mip-NeRF 360, Zip-NeRF), explicit splatting methods, and recent primitive-based radiance models. Table~\ref{tab:sota} reports results: \dbs{} achieves the best PSNR on Mip-NeRF 360 (28.70 dB), ties for best PSNR on NeRF Synthetic (34.53 dB), and remains competitive on the remaining metrics. Compared to its covariance-derived counterpart UBS~\citep{liu2025universal}, \dbs{} leads on PSNR and SSIM across both benchmarks and on LPIPS for NeRF Synthetic. These results assess competitiveness across metrics.

\textbf{Evaluation protocol.} UBS and \dbs{} in Table~\ref{tab:sota} are re-evaluated under the same 3DGS protocol described in \cref{sec:experiments}; UBS numbers therefore differ from those reported in the original UBS paper~\citep{liu2025universal}.

\begin{table*}[thp]
    \centering
    \caption{Comparison with SOTA methods on Mip-NeRF360~\citep{barron2022mipnerf360} and NeRF Synthetic~\citep{mildenhall2021nerf}. Literature baselines retain their reported protocols; UBS and \textit{d}BS use the same 3DGS evaluation protocol.}
    \resizebox{0.95\linewidth}{!}{%
    \begin{tabular}{cl|ccc|ccc}
        \toprule
        & \multirow{2}{*}{Methods}
        & \multicolumn{3}{c|}{Mip-NeRF360}
        & \multicolumn{3}{c}{NeRF Synthetic}\\
        \cmidrule(lr){3-8}
         &
         & PSNR$\uparrow$ & SSIM$\uparrow$ & LPIPS$\downarrow$
         & PSNR$\uparrow$ & SSIM$\uparrow$ & LPIPS$\downarrow$ \\
        \midrule
        \multirow{3}{*}{\rotatebox[origin=c]{90}{implicit}}
         & Instant-NGP~\citep{M_ller_2022}
           & 25.51 & 0.684 & 0.398
           &33.18 & 0.959&0.055\\
         & Mip-NeRF360~\citep{barron2022mipnerf360}
           & 27.69 & 0.792 & 0.237
           &33.25&0.962&0.039\\
         & Zip-NeRF~\citep{barron2023zipnerfantialiasedgridbasedneural}
           & 28.54 & 0.828 & 0.189
           &33.10&0.971&0.031\\
        \midrule
        \multirow{14}{*}{\rotatebox[origin=c]{90}{explicit}}
         & 3DGS~\citep{3dgs}
           & 27.20 & 0.815 & 0.214
           &33.31&0.969&0.037\\
         & GES~\citep{hamdi2024gesgeneralizedexponentialsplatting}
           &26.91&0.794&0.250
           & - & - & -  \\
         & 2DGS~\citep{2dgs}
           & 27.04 & 0.805 & 0.297
           &33.07&-&-\\
         & Mip-Splatting~\citep{yu2023mipsplattingaliasfree3dgaussian}
           & 27.79 & 0.827 & 0.203
           &33.33&0.969&0.039\\
         & 3DGS-MCMC~\citep{kheradmand20243dgaussiansplattingmarkov}
           & 28.29 & 0.840 & 0.210
           &33.80&0.970&0.040\\
         & DRKS \citep{huang2025deformable}
           &26.76&0.787&0.236
           &33.82&-&-\\
         &Textured GS \citep{chao2025textured}
           &27.35 & 0.827 & 0.186
           &33.24 & 0.967 & 0.043\\
         & Quadratic GS \citep{zhang2025quadratic}
           &27.39&0.813&0.213
           &-&-&-\\
         &DisC-GS \citep{qu2024disc}
           &28.01&0.833&0.189
           &-&-&-\\
         & Triangle Splatting~\citep{held2025triangle}
           &27.16&0.814&0.191
           &-&-&-\\
         & Radiance Meshes~\citep{radiancemesh}
           &27.16&0.819&0.255
           &-&-&-\\
         & Spherical Voronoi~\citep{di2025spherical}
           & \cellcolor{orange!25}28.57 & 0.835 & 0.230
           & \cellcolor{red!25}34.53 & \cellcolor{orange!25}0.973 & 0.032\\
         & UBS\citep{liu2025universal}
           & 28.50 & \cellcolor{orange!25}0.841 & \cellcolor{red!25}0.183
           & \cellcolor{orange!25}34.48 & \cellcolor{red!25}0.975 & \cellcolor{orange!25}0.026\\
         & \textit{d}BS (\textbf{Ours})
           & \cellcolor{red!25}28.70 & \cellcolor{red!25}0.842 & \cellcolor{orange!25}0.185
           & \cellcolor{red!25}34.53 & \cellcolor{red!25}0.975 & \cellcolor{red!25}0.025\\
        \bottomrule
    \end{tabular}}
    \label{tab:sota}
\end{table*}

\section{Conditional Slicing Algorithms}
\label{sec:algorithms}

The fused CUDA kernels for \dgs{} and \dbs{} are summarized below. Both share the same precision-Cholesky distance computation $\mathbf{L}^\top\boldsymbol{\delta}$ and the same factorized regression structure for position; they differ only in the opacity decay function applied to that distance.

\begin{algorithm}[h]
\caption{\dgs{} Conditional Slicing}
\label{alg:slice}
\begin{algorithmic}[1]
\REQUIRE Primitive parameters: $\mu_p$, $\mu_q$, $\alpha$, $\mathbf{S}$, $\boldsymbol{\Theta}_{pq}$, $\mathbf{L}$, $\boldsymbol{\Lambda}$
\REQUIRE Query point: $q$ (view/time)
\STATE $\boldsymbol{\delta} \gets q - \mu_q$
\STATE $\alpha_{\text{cond}} \gets \alpha \cdot \exp(-\lambda_o\,\|\mathbf{L}^\top \boldsymbol{\delta}\|^2)$ \Comment{Gaussian opacity decay}
\IF{\dgso{} variant}
    \STATE \return  $\alpha_{\text{cond}}, \mu_p$ \Comment{Early exit for opacity only}
\ENDIF
\STATE $\mathbf{V}_{pq} \gets \text{mean}(\text{diag}(\mathbf{S})) \cdot \boldsymbol{\Theta}_{pq}$
\STATE $\tilde{\mathbf{V}}_{qq} \gets \mathrm{diag}(\boldsymbol{\Lambda})\,\mathbf{L}\mathbf{L}^\top$
\STATE $\mu_{\text{cond}} \gets \mu_p + \mathbf{V}_{pq} \cdot \tilde{\mathbf{V}}_{qq} \cdot \boldsymbol{\delta}$
\STATE \return $\alpha_{\text{cond}}$, $\mu_{\text{cond}}$
\end{algorithmic}
\end{algorithm}

\begin{algorithm}[h]
\caption{\dbs{} Conditional Slicing}
\label{alg:slice_dbs}
\begin{algorithmic}[1]
\REQUIRE Primitive parameters: $\mu_p$, $\mu_q$, $\alpha$, $\mathbf{S}$, $\boldsymbol{\Theta}_{pq}$, $\mathbf{L}$, activated bandwidth $\boldsymbol{\beta}$
\REQUIRE Query point: $q$ (view/time)
\STATE $\boldsymbol{\delta} \gets q - \mu_q$
\STATE $\mathbf{z} \gets \mathbf{L}^\top \boldsymbol{\delta}$
\STATE $\alpha_{\text{cond}} \gets \alpha \cdot \prod_{i=1}^{C} (1 - \tanh(z_i^2))^{\beta_i}$ \Comment{Beta opacity decay}
\STATE $\boldsymbol{\Lambda}_\beta \gets \min(\boldsymbol{\beta}/4,\, \mathbf{1})$ \Comment{Beta-derived coupling}
\STATE $\mathbf{V}_{pq} \gets \text{mean}(\text{diag}(\mathbf{S})) \cdot \boldsymbol{\Theta}_{pq}$
\STATE $\tilde{\mathbf{V}}_{qq} \gets \mathrm{diag}(\boldsymbol{\Lambda}_\beta)\,\mathbf{L}\mathbf{L}^\top$
\STATE $\mu_{\text{cond}} \gets \mu_p + \mathbf{V}_{pq} \cdot \tilde{\mathbf{V}}_{qq} \cdot \boldsymbol{\delta}$
\STATE \return $\alpha_{\text{cond}}$, $\mu_{\text{cond}}$
\end{algorithmic}
\end{algorithm}

\section{Direct Conditional Coordinates for Joint Gaussian Primitives}
\label{sec:conditional_coordinates}

The operational quantities underlying the direct parameters define coordinates
on the space of joint Gaussian covariances. This establishes that the query
precision, query-to-position regression operator, and rendered spatial
covariance used by \dgs{} are compatible with a valid joint Gaussian, even
though \dgs{} never constructs that covariance. The unrestricted coordinate
map is a global diffeomorphism; \dgs{} directly optimizes a structured
factorization of these rendering-aligned coordinates.

Let the spatial dimension be $p=3$, let the query dimension be $C$, and let
\begin{equation}
\boldsymbol{\Sigma}_{\mathrm{cond}}\in\mathbb{S}_{++}^{p},\qquad
\mathbf{V}_{qq}\in\mathbb{S}_{++}^{C},\qquad
\mathbf{M}\in\mathbb{R}^{p\times C}.
\label{eq:conditional_coordinates_domain}
\end{equation}
Here $\boldsymbol{\Sigma}_{\mathrm{cond}}$ denotes the spatial covariance stored
and rendered by \dgs{}. In the reconstructed joint-Gaussian interpretation
below, it is the conditional spatial covariance and should not be confused with
the spatial marginal block $\boldsymbol{\Sigma}_{pp}^{\mathrm{joint}}$.
$\mathbf{V}_{qq}$ is the query-space precision, and $\mathbf{M}$ is the effective
query-to-position regression operator. In \dgs{}, these operational quantities are
\begin{equation}
\mathbf{V}_{qq}=\mathbf{L}\mathbf{L}^{\top},\qquad
\mathbf{M}=\mathbf{V}_{pq}\operatorname{diag}(\boldsymbol{\Lambda})
\mathbf{V}_{qq},
\label{eq:dgs_conditional_coordinates}
\end{equation}
while $\boldsymbol{\Sigma}_{\mathrm{cond}}$ remains query-independent.

\begin{proposition}[Conditional-coordinate parameterization]
\label{prop:conditional_coordinates}
Define
\begin{equation}
\boldsymbol{\Psi}(\boldsymbol{\Sigma}_{\mathrm{cond}},\mathbf{M},\mathbf{V}_{qq})
=
\begin{pmatrix}
\boldsymbol{\Sigma}_{\mathrm{cond}}+\mathbf{M}\mathbf{V}_{qq}^{-1}\mathbf{M}^{\top}
& \mathbf{M}\mathbf{V}_{qq}^{-1}\\
\mathbf{V}_{qq}^{-1}\mathbf{M}^{\top} & \mathbf{V}_{qq}^{-1}
\end{pmatrix}.
\label{eq:conditional_coordinate_map}
\end{equation}
Then:
\begin{enumerate}[label=(\roman*),leftmargin=*]
\item $\boldsymbol{\Psi}(\boldsymbol{\Sigma}_{\mathrm{cond}},\mathbf{M},\mathbf{V}_{qq})$
is symmetric positive-definite;
\item its query marginal has precision $\mathbf{V}_{qq}$;
\item conditioning at query $\mathbf{q}$ gives
\begin{equation}
\boldsymbol{\mu}_{p\mid q}
=\boldsymbol{\mu}_{p}+\mathbf{M}(\mathbf{q}-\boldsymbol{\mu}_{q}),
\qquad
\boldsymbol{\Sigma}_{p\mid q}=\boldsymbol{\Sigma}_{\mathrm{cond}};
\label{eq:conditional_coordinate_slice}
\end{equation}
\item the map
\begin{equation}
\boldsymbol{\Psi}:
\mathbb{S}_{++}^{p}\times\mathbb{R}^{p\times C}\times
\mathbb{S}_{++}^{C}\longrightarrow\mathbb{S}_{++}^{p+C}
\label{eq:conditional_coordinate_bijection}
\end{equation}
is a smooth bijection.
\end{enumerate}
\end{proposition}

\begin{proof}
The covariance in \cref{eq:conditional_coordinate_map} has the block-triangular
factorization
\begin{equation}
\boldsymbol{\Psi}
=
\begin{pmatrix}
\mathbf{I}_{p} & \mathbf{M}\\
\mathbf{0} & \mathbf{I}_{C}
\end{pmatrix}
\begin{pmatrix}
\boldsymbol{\Sigma}_{\mathrm{cond}} & \mathbf{0}\\
\mathbf{0} & \mathbf{V}_{qq}^{-1}
\end{pmatrix}
\begin{pmatrix}
\mathbf{I}_{p} & \mathbf{0}\\
\mathbf{M}^{\top} & \mathbf{I}_{C}
\end{pmatrix}.
\label{eq:conditional_coordinate_factorization}
\end{equation}
The outer factors are invertible and the middle factor is positive-definite, so
$\boldsymbol{\Psi}\succ\mathbf{0}$. Its blocks satisfy
\begin{equation}
\boldsymbol{\Sigma}_{qq}=\mathbf{V}_{qq}^{-1},\qquad
\boldsymbol{\Sigma}_{pq}=\mathbf{M}\mathbf{V}_{qq}^{-1},\qquad
\boldsymbol{\Sigma}_{pp}=\boldsymbol{\Sigma}_{\mathrm{cond}}
+\mathbf{M}\mathbf{V}_{qq}^{-1}\mathbf{M}^{\top}.
\label{eq:conditional_coordinate_blocks}
\end{equation}
Consequently,
\begin{equation}
\boldsymbol{\Sigma}_{pq}\boldsymbol{\Sigma}_{qq}^{-1}
=\mathbf{M}\mathbf{V}_{qq}^{-1}\mathbf{V}_{qq}=\mathbf{M},
\end{equation}
which gives the conditional mean in
\cref{eq:conditional_coordinate_slice}. The conditional covariance is
\begin{align}
\boldsymbol{\Sigma}_{p\mid q}
&=\boldsymbol{\Sigma}_{pp}
-\boldsymbol{\Sigma}_{pq}\boldsymbol{\Sigma}_{qq}^{-1}
 \boldsymbol{\Sigma}_{pq}^{\top}\\
&=\boldsymbol{\Sigma}_{\mathrm{cond}}+\mathbf{M}\mathbf{V}_{qq}^{-1}\mathbf{M}^{\top}
-\mathbf{M}\mathbf{V}_{qq}^{-1}\mathbf{V}_{qq}\mathbf{V}_{qq}^{-1}\mathbf{M}^{\top}
=\boldsymbol{\Sigma}_{\mathrm{cond}}.
\end{align}

Conversely, let
\begin{equation}
\boldsymbol{\Sigma}=
\begin{pmatrix}
\boldsymbol{\Sigma}_{pp} & \boldsymbol{\Sigma}_{pq}\\
\boldsymbol{\Sigma}_{pq}^{\top} & \boldsymbol{\Sigma}_{qq}
\end{pmatrix}\succ\mathbf{0}.
\end{equation}
Its conditional coordinates are uniquely recovered by
\begin{equation}
\mathbf{V}_{qq}=\boldsymbol{\Sigma}_{qq}^{-1},\qquad
\mathbf{M}=\boldsymbol{\Sigma}_{pq}\boldsymbol{\Sigma}_{qq}^{-1},\qquad
\boldsymbol{\Sigma}_{\mathrm{cond}}=\boldsymbol{\Sigma}_{pp}
-\boldsymbol{\Sigma}_{pq}\boldsymbol{\Sigma}_{qq}^{-1}
 \boldsymbol{\Sigma}_{pq}^{\top}.
\label{eq:conditional_coordinate_inverse}
\end{equation}
The Schur complement theorem gives $\boldsymbol{\Sigma}_{\mathrm{cond}}\succ\mathbf{0}$.
Both the forward map and this inverse are smooth on their positive-definite
domains, proving the claim.
\end{proof}

\paragraph{Parameter count.}
The dimensions match exactly:
\begin{equation}
\frac{p(p+1)}{2}+pC+\frac{C(C+1)}{2}
=\frac{(p+C)(p+C+1)}{2}.
\label{eq:conditional_coordinate_parameter_count}
\end{equation}
Thus unrestricted $(\boldsymbol{\Sigma}_{\mathrm{cond}},\mathbf{M},\mathbf{V}_{qq})$ are a
complete, nonredundant coordinate system for joint Gaussian covariances. The
actual \dgs{} factorization in \cref{eq:dgs_conditional_coordinates} additionally
constrains $\mathbf{V}_{pq}$ by spatial scaling and normalization and constrains
$\boldsymbol{\Lambda}\in[0,1]^C$. It therefore selects a structured subfamily of
this chart; we do not claim that the internal
$(\mathbf{V}_{pq},\boldsymbol{\Lambda})$ factorization is itself bijective.

\begin{corollary}[Exact implicit joint Gaussian for \dgs{}]
\label{cor:dgs_joint_gaussian}
Define the opacity precision and effective regression operator as
\begin{equation}
\boldsymbol{\Pi}=2\lambda_o\mathbf{V}_{qq},\qquad
\mathbf{M}=\mathbf{V}_{pq}\operatorname{diag}(\boldsymbol{\Lambda})
\mathbf{V}_{qq}.
\label{eq:dgs_exact_joint_coordinates}
\end{equation}
Then the implemented \dgs{} opacity exponent and position update are,
respectively, the query-marginal quadratic form and conditional-mean regression
of the valid joint covariance
\begin{equation}
\boxed{
\boldsymbol{\Sigma}_{\mathrm{dGS}}^{\star}
=\boldsymbol{\Psi}(\boldsymbol{\Sigma}_{\mathrm{cond}},\mathbf{M},\boldsymbol{\Pi})
=
\begin{pmatrix}
\boldsymbol{\Sigma}_{\mathrm{cond}}
+\mathbf{M}\boldsymbol{\Pi}^{-1}\mathbf{M}^{\top}
& \mathbf{M}\boldsymbol{\Pi}^{-1}\\
\boldsymbol{\Pi}^{-1}\mathbf{M}^{\top} & \boldsymbol{\Pi}^{-1}
\end{pmatrix}.}
\label{eq:dgs_implicit_joint_covariance}
\end{equation}
Equivalently, its blocks are
\begin{align}
\boldsymbol{\Sigma}_{qq}^{\star}
&=\frac{1}{2\lambda_o}\mathbf{V}_{qq}^{-1},
&
\boldsymbol{\Sigma}_{pq}^{\star}
&=\frac{1}{2\lambda_o}\mathbf{V}_{pq}
  \operatorname{diag}(\boldsymbol{\Lambda}),
\label{eq:dgs_exact_joint_blocks_query}\\
\boldsymbol{\Sigma}_{pp}^{\star}
&=\boldsymbol{\Sigma}_{\mathrm{cond}}
+\frac{1}{2\lambda_o}\mathbf{V}_{pq}
  \operatorname{diag}(\boldsymbol{\Lambda})\mathbf{V}_{qq}
  \operatorname{diag}(\boldsymbol{\Lambda})\mathbf{V}_{pq}^{\top}.
\label{eq:dgs_exact_joint_blocks_spatial}
\end{align}
The resulting conditional mean and covariance are exactly
\begin{equation}
\boldsymbol{\mu}_{\mathrm{cond}}
=\boldsymbol{\mu}_{p}
+\mathbf{V}_{pq}\operatorname{diag}(\boldsymbol{\Lambda})\mathbf{V}_{qq}
 (\mathbf{q}-\boldsymbol{\mu}_{q}),
\qquad
\boldsymbol{\Sigma}_{p\mid q}=\boldsymbol{\Sigma}_{\mathrm{cond}}.
\label{eq:dgs_joint_conditionals}
\end{equation}
\end{corollary}

\begin{proof}
Apply \cref{prop:conditional_coordinates} with query precision
$\boldsymbol{\Pi}$ and regression operator $\mathbf{M}$. The query marginal
has precision
$(\boldsymbol{\Sigma}_{qq}^{\star})^{-1}=\boldsymbol{\Pi}
=2\lambda_o\mathbf{V}_{qq}$, so its conventional Gaussian exponent is
\begin{equation}
-\frac{1}{2}\boldsymbol{\delta}^{\top}
(\boldsymbol{\Sigma}_{qq}^{\star})^{-1}\boldsymbol{\delta}
=-\lambda_o\boldsymbol{\delta}^{\top}\mathbf{V}_{qq}
 \boldsymbol{\delta},
\end{equation}
which is exactly the implemented opacity exponent. Moreover,
\begin{equation}
\boldsymbol{\Sigma}_{pq}^{\star}
(\boldsymbol{\Sigma}_{qq}^{\star})^{-1}
=\frac{1}{2\lambda_o}\mathbf{V}_{pq}
 \operatorname{diag}(\boldsymbol{\Lambda})
 (2\lambda_o\mathbf{V}_{qq})
=\mathbf{M},
\end{equation}
which is the implemented conditional-mean regression. The conditional covariance follows
from \cref{prop:conditional_coordinates}, and direct substitution gives the
blocks in \cref{eq:dgs_exact_joint_blocks_query,eq:dgs_exact_joint_blocks_spatial}.
\end{proof}

As in standard Gaussian splatting, opacity uses an unnormalized Gaussian kernel
rather than a normalized probability density. Consequently, the probabilistic
interpretation concerns the quadratic form, conditional mean, and covariance
structure; it does not require including the Gaussian density normalization
factor in the renderer.

\begin{corollary}[Decoupling and full coupling]
\label{cor:dgs_coupling_limits}
When $\boldsymbol{\Lambda}=\mathbf{0}$,
\cref{eq:dgs_implicit_joint_covariance} is block diagonal and position is
query-independent although opacity remains query-dependent. When
$\boldsymbol{\Lambda}=\mathbf{1}$, the full regression coupling is active,
$\mathbf{M}=\mathbf{V}_{pq}\mathbf{V}_{qq}$. Intermediate values continuously
weight the precision-transformed query channels.
\end{corollary}

\subsection{Determinant Structure}
\label{subsec:conditional_coordinate_determinant}

The outer factors in \cref{eq:conditional_coordinate_factorization} are unit
block-triangular. Therefore
\begin{equation}
\boxed{
\det\boldsymbol{\Psi}
=\det(\boldsymbol{\Sigma}_{\mathrm{cond}})\det(\mathbf{V}_{qq}^{-1})
=\frac{\det\boldsymbol{\Sigma}_{\mathrm{cond}}}{\det\mathbf{V}_{qq}}.}
\label{eq:conditional_coordinate_determinant}
\end{equation}
The regression operator $\mathbf{M}$ does not appear. Geometrically, $\mathbf{M}$ acts
as a block shear: it changes cross-covariance and the conditional mean
without changing total covariance volume. The two determinant factors separately
measure spatial conditional volume and query-space volume.

\subsection{Volume-Element Jacobian of the Coordinate Transformation}
\label{subsec:conditional_coordinate_jacobian}

Use the standard Lebesgue coordinates on symmetric and rectangular matrices.
The conditional-coordinate map has a closed-form volume-element Jacobian.

\begin{proposition}[Coordinate-volume factor]
\label{prop:conditional_coordinate_jacobian}
For the map in \cref{eq:conditional_coordinate_map},
\begin{equation}
\boxed{
\left|\det D\boldsymbol{\Psi}\right|
=(\det\mathbf{V}_{qq})^{-(p+C+1)}.}
\label{eq:conditional_coordinate_jacobian}
\end{equation}
For spatial dimension $p=3$, hence
\begin{equation}
\left|\det D\boldsymbol{\Psi}\right|
=(\det\mathbf{V}_{qq})^{-(C+4)}.
\label{eq:dgs_conditional_coordinate_jacobian}
\end{equation}
\end{proposition}

\begin{proof}
Factor the transformation into three block-triangular steps. First set
$\mathbf{Q}=\mathbf{V}_{qq}^{-1}$. Inversion on $\mathbb{S}_{++}^{C}$ has Jacobian
\begin{equation}
\left|\det\frac{\partial\mathbf{Q}}{\partial\mathbf{V}_{qq}}\right|
=(\det\mathbf{Q})^{C+1}=(\det\mathbf{V}_{qq})^{-(C+1)}.
\label{eq:spd_inverse_jacobian}
\end{equation}
Second, $\boldsymbol{\Sigma}_{pq}=\mathbf{M}\mathbf{Q}$. Right multiplication
of a $p\times C$ matrix by $\mathbf{Q}$ contributes
\begin{equation}
(\det\mathbf{Q})^{p}=(\det\mathbf{V}_{qq})^{-p}.
\label{eq:right_multiplication_jacobian}
\end{equation}
Finally,
$\boldsymbol{\Sigma}_{pp}=\boldsymbol{\Sigma}_{\mathrm{cond}}
+\mathbf{M}\mathbf{Q}\mathbf{M}^{\top}$ is a translation in the
$\boldsymbol{\Sigma}_{\mathrm{cond}}$ coordinates and contributes determinant one. Cross
derivatives between these steps lie below the corresponding diagonal Jacobian
blocks and do not change the determinant. Multiplying the diagonal factors gives
\cref{eq:conditional_coordinate_jacobian}.
\end{proof}

For every $\boldsymbol{\Sigma}_{\mathrm{cond}}\succ\mathbf{0}$ and
$\mathbf{V}_{qq}\succ\mathbf{0}$, the coordinate-volume factor is finite and
nonzero. Thus the conditional-coordinate map has no singularities in the
interior of its SPD domain. The domain retains the usual boundary behavior when
either SPD factor approaches its boundary, including when the conditional
covariance or corresponding query covariance approaches singularity. The
construction $\mathbf{V}_{qq}=\mathbf{L}\mathbf{L}^{\top}$ with a strictly
positive Cholesky diagonal keeps \dgs{} in the interior with respect to query
precision. This statement concerns the ambient
$(\boldsymbol{\Sigma}_{\mathrm{cond}},\mathbf{M},\mathbf{V}_{qq})$ chart, not
uniqueness of \dgs{}'s structured internal factorization of $\mathbf{M}$.

\paragraph{Interpretation.}
The rendered spatial covariance, query precision, and regression operator are
Gaussian regression coordinates. Direct conditioning optimizes these operational
quantities rather than repeatedly recovering them from an ambient covariance by
inversion and Schur complementation. The coupling is a block shear with an
explicit inverse and determinant. Within the structured dGS factorization,
$\mathbf{V}_{pq}$ supplies the spatial-scale prior, while
$\operatorname{diag}(\boldsymbol{\Lambda})\mathbf{V}_{qq}$ completes the
regression operator and exposes learnable coupling.

\section{Covariance-Derived versus Direct Regression Coordinates}
\label{sec:spatial_scaling_proof}

The conditional mean depends on the regression operator
$\boldsymbol{\Sigma}_{pq}\boldsymbol{\Sigma}_{qq}^{-1}$, not on the
cross-covariance $\boldsymbol{\Sigma}_{pq}$ alone. This section derives that
operator for N-DGS~\citep{gao20246dgs,gao20257dgs} and contrasts its implicit
covariance coordinates with the direct conditional coordinates of \dgs{}.
UBS~\citep{liu2025universal} inherits the same covariance-derived structure.

\paragraph{Block-Cholesky derivation.}
Partition the joint Cholesky factor into spatial and query blocks:
\begin{equation}
\mathcal{L}=
\begin{pmatrix}
\mathbf{A} & \mathbf{0}\\
\mathbf{B} & \mathbf{D}
\end{pmatrix},
\qquad
\mathbf{A}\in\mathbb{R}^{p\times p},\quad
\mathbf{B}\in\mathbb{R}^{C\times p},\quad
\mathbf{D}\in\mathbb{R}^{C\times C}.
\label{eq:ndgs_block_cholesky}
\end{equation}
Here $\mathbf{A}$ and $\mathbf{D}$ are lower triangular. Multiplication gives
\begin{equation}
\boldsymbol{\Sigma}
=\mathcal{L}\mathcal{L}^{\top}
=
\begin{pmatrix}
\mathbf{A}\mathbf{A}^{\top} & \mathbf{A}\mathbf{B}^{\top}\\
\mathbf{B}\mathbf{A}^{\top}
& \mathbf{B}\mathbf{B}^{\top}+\mathbf{D}\mathbf{D}^{\top}
\end{pmatrix},
\label{eq:ndgs_block_covariance}
\end{equation}
and therefore
\begin{equation}
\boldsymbol{\Sigma}_{pp}=\mathbf{A}\mathbf{A}^{\top},\qquad
\boldsymbol{\Sigma}_{pq}=\mathbf{A}\mathbf{B}^{\top},\qquad
\boldsymbol{\Sigma}_{qq}
=\mathbf{B}\mathbf{B}^{\top}+\mathbf{D}\mathbf{D}^{\top}.
\label{eq:ndgs_covariance_blocks}
\end{equation}
The actual N-DGS query-to-position regression operator is
\begin{equation}
\boxed{
\mathbf{M}_{\mathrm{N\text{-}DGS}}
=
\mathbf{A}\mathbf{B}^{\top}
\left(\mathbf{B}\mathbf{B}^{\top}
+\mathbf{D}\mathbf{D}^{\top}\right)^{-1}.}
\label{eq:ndgs_regression_operator}
\end{equation}
Thus $\mathbf{A}$ affects both spatial covariance and cross-covariance;
$\mathbf{B}$ affects both cross-covariance and query covariance; and
$\mathbf{D}$ also affects query covariance. The operational regression is
available only after forming and inverting the query block.

\paragraph{A valid cross-covariance bound.}
For spatial row $i$,
\begin{equation}
(\boldsymbol{\Sigma}_{pq})_{i,:}
=\mathbf{A}_{i,:}\mathbf{B}^{\top},
\qquad
\|(\boldsymbol{\Sigma}_{pq})_{i,:}\|_2
\leq \|\mathbf{A}_{i,:}\|_2\,\|\mathbf{B}\|_2.
\label{eq:ndgs_cross_covariance_row_bound}
\end{equation}
Moreover,
\begin{equation}
\|\mathbf{A}_{i,:}\|_2^2
=(\mathbf{A}\mathbf{A}^{\top})_{ii}
=(\boldsymbol{\Sigma}_{pp})_{ii}.
\end{equation}
In N-DGS, bounded Cholesky coefficients are multiplied by the query
diagonal scales $\mathbf{s}_q$. Thus
$\mathbf{B}=\operatorname{diag}(\mathbf{s}_q)\widetilde{\mathbf{B}}$
with $|\widetilde{B}_{jk}|\leq1$, giving
\begin{equation}
\|\mathbf{B}\|_2\leq\|\mathbf{B}\|_F\leq\sqrt{p}\,\|\mathbf{s}_q\|_2,
\end{equation}
so
\begin{equation}
\boxed{
\|(\boldsymbol{\Sigma}_{pq})_{i,:}\|_2
\leq \sqrt{p}\,\|\mathbf{s}_q\|_2\,
\sqrt{(\boldsymbol{\Sigma}_{pp})_{ii}}.}
\label{eq:ndgs_cross_covariance_invariant_bound}
\end{equation}
The off-diagonal entries of $\mathbf{A}$ contribute both to
$\boldsymbol{\Sigma}_{pq}$ and to the actual marginal spatial extent
$(\boldsymbol{\Sigma}_{pp})_{ii}$. They should therefore not be interpreted as
producing cross-covariance independently of spatial scale.

\paragraph{Cross-covariance is not displacement.}
The rendered N-DGS shift is
\begin{equation}
\Delta\boldsymbol{\mu}_{\mathrm{N\text{-}DGS}}
=
\boldsymbol{\Sigma}_{pq}\boldsymbol{\Sigma}_{qq}^{-1}
\boldsymbol{\delta}
=
\mathbf{M}_{\mathrm{N\text{-}DGS}}\boldsymbol{\delta}.
\label{eq:ndgs_position_shift}
\end{equation}
A bound on $\boldsymbol{\Sigma}_{pq}$ alone therefore does not bound the shift:
the inverse query covariance and the query deviation also enter. In Euclidean
coordinates, a narrow or ill-conditioned query covariance can amplify the
regression operator.

\paragraph{Invariant Gaussian bound.}
A basis-independent statement follows from the normalized cross-covariance
\begin{equation}
\mathbf{K}
=
\boldsymbol{\Sigma}_{pp}^{-1/2}
\boldsymbol{\Sigma}_{pq}
\boldsymbol{\Sigma}_{qq}^{-1/2}.
\label{eq:canonical_correlation_matrix}
\end{equation}
Positive-definiteness of the joint covariance implies
$\|\mathbf{K}\|_2<1$ by the Schur complement. With the whitened query deviation
$\mathbf{h}=\boldsymbol{\Sigma}_{qq}^{-1/2}\boldsymbol{\delta}$,
\begin{equation}
\Delta\boldsymbol{\mu}_{\mathrm{N\text{-}DGS}}
=
\boldsymbol{\Sigma}_{pp}^{1/2}\mathbf{K}\mathbf{h},
\end{equation}
and hence
\begin{equation}
\boxed{
\left\|
\boldsymbol{\Sigma}_{pp}^{-1/2}
\Delta\boldsymbol{\mu}_{\mathrm{N\text{-}DGS}}
\right\|_2
\leq
\left\|
\boldsymbol{\Sigma}_{qq}^{-1/2}\boldsymbol{\delta}
\right\|_2.}
\label{eq:ndgs_invariant_shift_bound}
\end{equation}
Measured in marginal spatial standard-deviation units, the shift is bounded
by the query deviation in query standard-deviation units, irrespective of
basis or Cholesky ordering.

\paragraph{Direct regression coordinates in \dgs{}.}
The direct formulation parameterizes
\begin{equation}
\mathbf{V}_{qq}=\mathbf{L}\mathbf{L}^{\top},\qquad
\mathbf{M}_{\mathrm{dGS}}
=
\mathbf{V}_{pq}\operatorname{diag}(\boldsymbol{\Lambda})
\mathbf{V}_{qq},\qquad
\boldsymbol{\Sigma}_{\mathrm{cond}}
\ \text{independently}.
\label{eq:dgs_direct_regression_coordinates}
\end{equation}
Its spatial coupling factor is normalized as
\begin{equation}
\mathbf{V}_{pq}=\bar{s}\boldsymbol{\Theta}_{pq},\qquad
\|\boldsymbol{\Theta}_{pq}\|_F=1,
\end{equation}
which gives
\begin{equation}
\|\mathbf{V}_{pq}\|_F=\bar{s},
\qquad
\|\mathbf{V}_{pq}^{(i)}\|_2\leq\bar{s}.
\label{eq:dgs_coupling_factor_bound}
\end{equation}
This is an explicit spatial-scale prior on the coupling factor, not a bound on
the effective regression or rendered shift. Those additionally depend on
$\operatorname{diag}(\boldsymbol{\Lambda})\mathbf{V}_{qq}$ and
$\boldsymbol{\delta}$.

The conditional-coordinate theorem in
Appendix~\ref{sec:conditional_coordinates} makes the comparison precise. N-DGS
indirectly parameterizes conditional covariance, regression, and query precision
through the ambient joint Cholesky factor. \dgs{} makes the corresponding
operational quantities explicit and exposes a learnable coupling strength,
without forming a query-block inverse during slicing. The benefit is this
computational and parameterization separation, rather than an invalid geometric
bound for N-DGS.

Empirically, removing the spatial factor $\bar{s}$ degrades PSNR by 2.65 dB on
NeRF Synthetic and 2.97 dB on 6DGS-PBR
(Table~\ref{tab:ablation}), supporting the spatial-scale prior on
$\mathbf{V}_{pq}$.
On two trained MCMC scenes with six views each, opacity-weighted
$\|\Delta\boldsymbol{\mu}\|/\bar{s}$ has median 0.29--0.30 and 99th
percentile 4.8--7.0 for N-DGS, versus 2.9--4.8 at the 99th percentile for
\dgs{}. The ordering reverses when normalized by the smallest principal
axis, supporting $\bar{s}$ as an isotropic prior rather than a per-axis optimum.

\section{Results with Standard Densification}
\label{sec:detailed_results}

The main paper reports MCMC densification~\citep{kheradmand20243dgaussiansplattingmarkov}, the protocol used by UBS originally. Here we additionally evaluate the Gaussian-kernel pair (\dgs{} vs.\ N-DGS) under standard 3DGS densification~\citep{3dgs}, the protocol used by N-DGS originally. For dynamic scenes, these standard-densification runs use batch size 4, following the original N-DGS protocol. Table~\ref{tab:summary_results} summarizes quality (PSNR, SSIM, LPIPS) and efficiency metrics across all five datasets: three static (NeRF Synthetic~\citep{mildenhall2021nerf}, Mip-NeRF 360~\citep{barron2022mipnerf360}, 6DGS-PBR~\citep{gao20246dgs}) and two dynamic (D-NeRF~\citep{pumarola2020dnerfneuralradiancefields}, 7DGS-PBR~\citep{gao20257dgs}).

\begin{table*}[h]
\centering
\caption{Standard 3DGS densification: scene means from one run per scene (\#Gauss in K, training time in minutes; dynamic batch size 4). FPS$_{\text{CUDA}}$ is an arithmetic scene mean. For 3DGS~\citep{3dgs} and 4DGS~\citep{yang2023real} comparisons, see the original N-DGS~\citep{gao20246dgs,gao20257dgs}.}
\label{tab:summary_results}
\resizebox{0.85\textwidth}{!}{%
\begin{tabular}{@{}ccl|rrr|ccc@{}}
\toprule
& Dataset & Method & \#Gauss & Time $\downarrow$ & FPS$_{\text{CUDA}}$ $\uparrow$ & PSNR $\uparrow$ & SSIM $\uparrow$ & LPIPS $\downarrow$ \\
\midrule
\multirow{9}{*}{\rotatebox{90}{STATIC}}
& \multirow{3}{*}{\makecell{\texttt{NeRF}\\\texttt{Synthetic}}}
& N-DGS & \bf 171 & 6.35 & 615.57 & 33.38 & 0.970 & 0.032 \\
&& \textit{d}GS-O (\textbf{Ours}) & 201 & \bf 5.83 & \bf 754.98 & 33.84 & 0.971 & 0.030 \\
&& \textit{d}GS (\textbf{Ours}) & 201 & 6.84 & 704.84 & \bf 33.91 & \bf 0.972 & \bf 0.029 \\
\cmidrule{2-9}
& \multirow{3}{*}{\makecell{\texttt{Mip-NeRF}\\\texttt{360}}}
& N-DGS & \bf 1273 & 27.09 & 147.46 & 26.40 & 0.739 & 0.316 \\
&& \textit{d}GS-O (\textbf{Ours}) & 2363 & \bf 25.58 & \bf 279.30 & 27.18 & 0.779 & 0.252 \\
&& \textit{d}GS (\textbf{Ours}) & 2463 & 28.21 & 237.82 & \bf 27.37 & \bf 0.781 & \bf 0.251 \\
\cmidrule{2-9}
& \multirow{3}{*}{\texttt{6DGS-PBR}}
& N-DGS & \bf 84 & \bf 26.77 & 678.40 & 37.11 & 0.971 & 0.054 \\
&& \textit{d}GS-O (\textbf{Ours}) & 118 & 30.72 & \bf 696.31 & 36.54 & 0.965 & 0.059 \\
&& \textit{d}GS (\textbf{Ours}) & 134 & 33.80 & 653.44 & \bf 37.75 & \bf 0.972 & \bf 0.052 \\

\midrule
\midrule
\multirow{6}{*}{\rotatebox{90}{DYNAMIC}}
& \multirow{3}{*}{\texttt{D-NeRF}}
& N-DGS & \bf 69 & 14.08 & 852.01 & 33.13 & 0.968 & 0.038 \\
&& \textit{d}GS-O (\textbf{Ours}) & 70 & \bf 13.20 & \bf 948.77 & 30.92 & 0.959 & 0.044 \\
&& \textit{d}GS (\textbf{Ours}) & 85 & 15.56 & 818.41 & \bf 33.59 & \bf 0.971 & \bf 0.032 \\
\cmidrule{2-9}
& \multirow{3}{*}{\texttt{7DGS-PBR}}
& N-DGS & \bf 162 & 38.14 & 709.83 & 31.75 & 0.955 & 0.055 \\
&& \textit{d}GS-O (\textbf{Ours}) & 227 & \bf 34.58 & \bf 835.24 & 31.70 & 0.955 & 0.053 \\
&& \textit{d}GS (\textbf{Ours}) & 226 & 36.96 & 753.77 & \bf 32.18 & \bf 0.956 & \bf 0.050 \\
\bottomrule
\end{tabular}}
\vspace{-0.5em}
\end{table*}

\textbf{Static Scenes.}
On NeRF Synthetic, both variants improve quality: \textit{d}GS-O by +0.46 dB and \textit{d}GS by +0.53 dB, with rendering speedups of 23\% and 15\% respectively. On Mip-NeRF 360, efficiency gains are most pronounced: \textit{d}GS-O achieves 89\% faster rendering with +0.78 dB improvement, while \textit{d}GS achieves 61\% faster rendering with +0.97 dB improvement. Notably, \textit{d}GS-O uses 86\% more primitives (2363K vs 1273K) yet remains faster, demonstrating that per-primitive computation cost dominates over primitive count. On 6DGS-PBR, \textit{d}GS-O shows modest 3\% speedup but -0.57 dB quality drop, while \textit{d}GS recovers and exceeds baseline quality (+0.64 dB).

\textbf{Dynamic Scenes.}
Position conditioning proves essential for temporal dynamics. On D-NeRF, \textit{d}GS-O suffers -2.21 dB quality drop despite 11\% faster rendering, while \textit{d}GS with learnable coupling via $\Lambda$ exceeds N-DGS by +0.46 dB, demonstrating adaptive coupling surpasses fixed coupling. On 7DGS-PBR combining dynamics with angular lighting, \textit{d}GS-O achieves 18\% faster rendering with minimal quality loss (-0.05 dB), while \textit{d}GS improves quality by +0.43 dB with 6\% faster rendering.

\begin{table*}[bp]
  \centering
  \caption{\small{Detailed quantitative comparison across all datasets at 30K iterations. N-DGS, dGS-O, and dGS are compared on static benchmarks (NeRF Synthetic, Mip-NeRF 360, 6DGS-PBR) and dynamic benchmarks (7DGS-PBR, D-NeRF).}} %
  \resizebox{1.0\linewidth}{!}{
        \begin{tabular}{@{\hskip 2pt}cc|@{\hskip 2pt}l@{\hskip 2pt}|@{\hskip 2pt}r@{\hskip 2pt}r@{\hskip 2pt}r|@{\hskip 2pt}r@{\hskip 2pt}r@{\hskip 2pt}r|@{\hskip 2pt}r@{\hskip 2pt}r@{\hskip 2pt}r@{\hskip 2pt}}

    \toprule

    \multicolumn{2}{@{\hskip 2pt}c|@{\hskip 2pt}}{\multirow{2}[4]{*}{Dataset}} & \multicolumn{1}{c|@{\hskip 2pt}}{\multirow{2}[4]{*}{Scene}}
        & \multicolumn{3}{c|@{\hskip 2pt}}{N-DGS}
        & \multicolumn{3}{c|@{\hskip 2pt}}{\textit{d}GS-O (\textbf{Ours})}
        & \multicolumn{3}{@{\hskip 2pt}c}{\textit{d}GS (\textbf{Ours})} \\
\cmidrule{4-12}
          &&       & \multicolumn{1}{@{\hskip 2pt}c@{\hskip 2pt}}{PSNR$\uparrow$}
                  & \multicolumn{1}{@{\hskip 2pt}c@{\hskip 2pt}}{SSIM$\uparrow$}
                  & \multicolumn{1}{c|@{\hskip 2pt}}{LPIPS$\downarrow$}
                  & \multicolumn{1}{@{\hskip 2pt}c@{\hskip 2pt}}{PSNR$\uparrow$}
                  & \multicolumn{1}{@{\hskip 2pt}c@{\hskip 2pt}}{SSIM$\uparrow$}
                  & \multicolumn{1}{c|@{\hskip 2pt}}{LPIPS$\downarrow$}
                  & \multicolumn{1}{@{\hskip 2pt}c@{\hskip 2pt}}{PSNR$\uparrow$}
                  & \multicolumn{1}{@{\hskip 2pt}c@{\hskip 2pt}}{SSIM$\uparrow$}
                  & \multicolumn{1}{@{\hskip 2pt}c@{\hskip 2pt}}{LPIPS$\downarrow$} \\
    \midrule

\multicolumn{1}{@{\hskip 2pt}c}{\multirow{30}[4]{*}{\begin{sideways}STATIC SCENES\end{sideways}}}

& \multicolumn{1}{c|@{\hskip 2pt}}{\multirow{9}[4]{*}{\begin{sideways}NeRF Synthetic\end{sideways}}}
    & \texttt{chair}     & 35.75 & 0.986 & 0.013 & \bf 36.36 & \bf 0.988 & \bf 0.012 &  36.26 & \bf 0.988 & \bf 0.012 \\
    & & \texttt{drums}     & 26.75 & 0.956 & 0.036 & 26.80 & 0.957 & 0.035 & \bf 26.92 & \bf 0.958 & \bf 0.034 \\
    & & \texttt{ficus}     & 33.54 & 0.984 & 0.014 & \bf 35.30 & \bf 0.988 & \bf 0.011 & 35.03 & 0.988 & \bf 0.011 \\
    & & \texttt{hotdog}    & 38.10 & 0.985 & 0.022 & \bf 38.13 & \bf 0.986 & \bf 0.020 & \bf 38.15 & \bf 0.986 & \bf 0.020 \\
    & & \texttt{lego}      & 35.63 & 0.981 & 0.019 & \bf 35.87 & \bf 0.982 & \bf 0.017 & 35.79 & \bf 0.982 & \bf 0.017 \\
    & & \texttt{materials} & 30.78 & 0.967 & 0.033 & 31.00 & 0.967 & 0.030 & \bf 31.76 & \bf 0.972 & \bf 0.027 \\
    & & \texttt{mic}       & 35.63 & 0.992 & 0.006 & 36.12 & \bf 0.993 & \bf 0.005 & \bf 36.25 & \bf 0.993 & \bf 0.005 \\
    & & \texttt{ship}      & 30.85 & 0.906 & 0.112 & \bf 31.16 & \bf 0.908 & 0.107 & 31.13 & \bf 0.908 & \bf 0.106 \\
\cmidrule{3-12}
    & & \texttt{avg}
      & 33.38
      & 0.970
      & 0.032
      & 33.84
      & 0.971
      & 0.030
      & \bf 33.91
      & \bf 0.972
      & \bf 0.029 \\
    \cmidrule{2-12}

& \multicolumn{1}{c|@{\hskip 2pt}}{\multirow{10}[4]{*}{\begin{sideways}Mip-NeRF 360\end{sideways}}}
        & \texttt{bicycle}   & 22.33 & 0.554 & 0.447 & 23.49 & 0.644 & 0.327 & \bf 23.51 & \bf 0.648 & \bf 0.318 \\
        & & \texttt{bonsai}    & 33.20 & 0.940 & 0.208 & 33.22 & 0.949 & 0.180 & \bf 33.46 & \bf 0.951 & \bf 0.175 \\
        & & \texttt{counter}   & 29.26 & 0.895 & 0.236 & 29.73 & 0.917 &  0.185 & \bf 29.93 & \bf 0.920 & \bf 0.181 \\
        & & \texttt{flowers}   & 18.90 & 0.409 & 0.507 &  19.61 & \bf 0.453 & \bf 0.443 & \bf 19.71 & 0.452 & 0.455 \\
        & & \texttt{garden}    & 26.19 & 0.809 & 0.205 & 26.77 & 0.842 &  0.145 & \bf 27.13 & \bf 0.849 & \bf 0.137 \\
        & & \texttt{kitchen}   & 31.17 & 0.922 & 0.137 & 31.72 & 0.930 &  0.121 & \bf 32.04 & \bf 0.933 & \bf 0.116 \\
        & & \texttt{room}      & 31.60 & 0.910 & 0.243 & 31.83 & 0.927 & 0.199 & \bf 32.25 & \bf 0.929 & \bf 0.195 \\
        & & \texttt{stump}     & 23.32 & 0.638 & 0.413 & \bf 26.01 &  0.750 & \bf 0.275 & 26.00 & \bf 0.742 & 0.295 \\
        & & \texttt{treehill}  & 21.66 & 0.571 & 0.448 & 22.21 & 0.603 & 0.393 & \bf 22.31 & \bf 0.606 & \bf 0.389 \\
    \cmidrule{3-12}
        & & \texttt{avg}
          & 26.40
          & 0.739
          & 0.316
          & 27.18
          & 0.779
          & 0.252
          & \bf 27.37
          & \bf 0.781
          & \bf 0.251 \\
\cmidrule{2-12}

& \multicolumn{1}{c|@{\hskip 2pt}}{\multirow{8}[4]{*}{\begin{sideways}6DGS-PBR\end{sideways}}}
    & \texttt{bunny} & 40.87 & 0.993 & 0.045 & 43.44 & 0.993 & 0.042 & \bf 44.14 & \bf 0.993 & \bf 0.040 \\
    & & \texttt{cloud}     & 41.93 & 0.991 & 0.052 & 43.47 & 0.992 & 0.048 & \bf 43.53 & \bf 0.992 & \bf 0.045 \\
    & & \texttt{dragon}    &  35.69 & 0.947 & 0.039 & 34.42 & 0.914 & 0.051 &  \bf 36.68 &  \bf 0.947 & \bf 0.039 \\
    & & \texttt{explosion} & \bf 40.97 & \bf 0.989 & \bf 0.037 & 37.81 & 0.983 & 0.041 & 38.95 & 0.985 & 0.039 \\
    & & \texttt{ct-scan}     & 33.48 & 0.964 & 0.057 & 31.81 & 0.956 & 0.067 & \bf 34.62 & \bf 0.970 & \bf 0.048 \\
    & & \texttt{smoke}     & \bf 40.02 & \bf 0.990 & \bf 0.041 & 38.20 & 0.989 & 0.049 & 39.28 & \bf 0.990 & 0.047 \\
    & & \texttt{suzanne}   & 26.81 & 0.929 & \bf 0.106 & 26.65 & 0.927 & 0.112 & \bf 27.03 & \bf 0.930 & 0.107 \\
\cmidrule{3-12}
    & & \texttt{avg}
      & 37.11
      & 0.971
      & 0.054
      & 36.54
      & 0.965
      & 0.059
      & \bf 37.75
      & \bf 0.972
      & \bf 0.052 \\

    \midrule
    \midrule

\multicolumn{1}{@{\hskip 2pt}c}{\multirow{16}[4]{*}{\begin{sideways}DYNAMIC SCENES\end{sideways}}}

& \multirow{7}{*}{\begin{sideways}7DGS-PBR\end{sideways}}
    & \texttt{cloud}    & 29.28 & 0.955 & 0.075 & 29.34 & 0.955 & 0.075 & \bf 29.82 & \bf 0.955 & \bf 0.074 \\
    & & \texttt{dust}     & 36.85 & 0.955 & 0.038 & 36.87 & 0.955 &  0.037 & \bf 38.39 & \bf 0.958 & \bf 0.031 \\
    & & \texttt{flame}    & \bf 31.64 & \bf 0.937 & \bf 0.062 & 30.38 & 0.933 & 0.066 & 31.52 & 0.936 & 0.060 \\
    & & \texttt{heart1}   & 34.66 & 0.983 & 0.023 & 35.14 & \bf 0.985 & 0.021 &  \bf 35.25 & \bf 0.986 & \bf 0.018 \\
    & & \texttt{heart2}   & 30.99 & 0.959 & 0.057 & \bf 31.29 & \bf 0.962 & 0.052 & 31.01 & \bf 0.962 & \bf 0.049 \\
    & & \texttt{suzanne}  & 27.09 & \bf 0.941 & 0.072 & \bf 27.20 & 0.940 & \bf 0.069 & 27.07 & 0.940 & 0.070 \\
    \cmidrule{3-12}
        & & \texttt{avg}
          & 31.75
          & 0.955
          & 0.055
          & 31.70
          & 0.955
          & 0.053
          & \bf 32.18
          & \bf 0.956
          & \bf 0.050 \\
    \cmidrule{2-12}

& \multirow{9}{*}{\begin{sideways}D-NeRF\end{sideways}}
        & \texttt{b.balls}  & \bf 34.00 & \bf 0.981 & 0.025 & 32.52 & 0.978 & 0.028 &  33.59 & \bf 0.981 & \bf 0.023 \\
        & & \texttt{h.warrior}& 32.76 & 0.932 & 0.093 & 33.54 & 0.941 & 0.085 & \bf 34.64 & \bf 0.947 & \bf 0.073 \\
        & & \texttt{hook}     & 31.15 & 0.960 & 0.044 & 28.55 & 0.937 & 0.062 & \bf 32.94 & \bf 0.970 & \bf 0.030 \\
        & & \texttt{j.jacks}  & \bf 31.22 & \bf 0.967 & 0.043 & 28.62 & 0.952 & 0.054 & 30.86 & 0.965 & \bf 0.039 \\
        & & \texttt{lego}     & \bf 28.33 & \bf 0.943 & 0.055 & 27.13 & 0.939 &  0.055 & 28.06 & 0.942 & \bf 0.054 \\
        & & \texttt{mutant}   & 39.38 & 0.993 & 0.007 & 34.55 & 0.981 & 0.019 & \bf 39.84 & \bf 0.993 & \bf 0.006 \\
        & & \texttt{standup}  & 38.50 & 0.988 & 0.013 & 33.02 & 0.972 & 0.029 & \bf 38.76 & \bf 0.990 & \bf 0.009 \\
        & & \texttt{trex}     & 29.68 & 0.978 & 0.023 & 29.42 & 0.975 & 0.022 & \bf 30.08 & \bf 0.978 & \bf 0.021 \\
    \cmidrule{3-12}
        & & \texttt{avg}
          & 33.13
          & 0.968
          & 0.038
          & 30.92
          & 0.959
          & 0.044
          & \bf 33.59
          & \bf 0.971
          & \bf 0.032 \\
    \bottomrule
    \end{tabular}
    }
  \label{tab:detailed_combined}
\vspace{-1em}
\end{table*}

\section{UBS Formulation}
\label{sec:dbs}

The released UBS implementation used in our experiments evaluates
$\beta_i=4\exp(\beta_i^{\mathrm{raw}})$ and
$\boldsymbol{\Lambda}_\beta=\min(\boldsymbol{\beta}/4,\mathbf{1})$.
Its covariance-derived slicing is
\begin{align}
\mathbf{z}^{\text{UBS}} &= \hat{\mathbf{L}}_q^{-1} \boldsymbol{\delta}, \quad d_i = \tanh\!\bigl((z_i^{\text{UBS}})^2\bigr), \\
\alpha_{\text{cond}} &= \alpha \prod_{i=1}^{C} (1 - d_i)^{\beta_i}, \\
\boldsymbol{\mu}_{\text{cond}} &= \boldsymbol{\mu}_p + \boldsymbol{\Sigma}_{pq} \boldsymbol{\Sigma}_{qq}^{-1} \operatorname{diag}(\boldsymbol{\Lambda}_\beta) \boldsymbol{\delta}, \\
\boldsymbol{\Sigma}_{\text{cond}} &= \boldsymbol{\Sigma}_{pp} - \boldsymbol{\Sigma}_{pq} \boldsymbol{\Sigma}_{qq}^{-1} \operatorname{diag}(\boldsymbol{\Lambda}_\beta) \boldsymbol{\Sigma}_{qp},
\end{align}
where $\hat{\mathbf{L}}_q\hat{\mathbf{L}}_q^\top=\boldsymbol{\Sigma}_{qq}$.
The \texttt{tanh} values are capped below one for finite opacity evaluation.
\dbs{} (\cref{eq:dbs_alpha,eq:dbs_mu}) retains the same opacity decay,
bandwidth activation, and clipped coupling. It replaces the triangular solve
with $\mathbf{L}^\top\boldsymbol{\delta}$, replaces the covariance-derived
regression with $\mathbf{V}_{pq}\operatorname{diag}(\boldsymbol{\Lambda}_\beta)\mathbf{V}_{qq}$,
and directly stores $\boldsymbol{\Sigma}_{\mathrm{cond}}$.

\section{Parameterization and Evaluation Controls}
\label{app:robustness}

\textbf{Unrestricted direct regression.}
To distinguish direct readout from its structured parameterization, we replace
$\mathbf{V}_{pq}\operatorname{diag}(\boldsymbol{\Lambda})\mathbf{V}_{qq}$
with a free matrix $\mathbf{M}\in\mathbb{R}^{3\times C}$, retaining the directly
stored conditional covariance and full query precision. This
Direct-Unrestricted (DU) variant removes spatial normalization and bounded
coupling. Its matrix learning rate is $5\times10^{-6}$, scaled by the median
primitive size to account for the different parameter units and then fixed
across scenes.
Table~\ref{tab:du_control} shows that unrestricted direct coordinates alone do
not explain the quality gains: the structured parameterization performs better
on both datasets. This comparison evaluates the structure and optimization
coordinates collectively.

\begin{table}[!htbp]
\centering
\caption{Direct-regression control under standard densification. Dataset-mean
PSNR (dB), one training run per scene.}
\label{tab:du_control}
\begin{tabular}{lrrr}
\toprule
Dataset & N-DGS & DU & \dgs{} \\
\midrule
NeRF Synthetic & 33.38 & 32.91 & \textbf{33.91} \\
6DGS-PBR & 37.11 & 36.44 & \textbf{37.75} \\
\bottomrule
\end{tabular}
\vspace{-1em}
\end{table}

\textbf{Variation across training seeds.}
Table~\ref{tab:seed_controls} summarizes three matched seeds for the smallest
dBS/UBS margin and for the coupling ablation. The dBS/UBS difference remains
quality parity, while learned coupling ranks above both fixed endpoints in
all three runs. These repeated-seed results concern NeRF Synthetic; the other
benchmark comparisons retain the single-run protocol.

\begin{table}[!htbp]
\centering
\caption{NeRF Synthetic PSNR (dB): mean and sample standard deviation of
dataset means across three training seeds.}
\label{tab:seed_controls}
\begin{tabular}{llr}
\toprule
Protocol & Method & PSNR \\
\midrule
MCMC & UBS & $34.50\pm0.02$ \\
& \dbs{} & $34.53\pm0.02$ \\
\midrule
Standard & $\boldsymbol{\Lambda}=\mathbf{0}$ & $33.83\pm0.02$ \\
& $\boldsymbol{\Lambda}=\mathbf{1}$ & $33.78\pm0.01$ \\
& Learned $\boldsymbol{\Lambda}$ & $33.92\pm0.02$ \\
\bottomrule
\end{tabular}
\vspace{-1em}
\end{table}

\textbf{Batch size.}
On 7DGS-PBR, using batch size four for both Beta methods
gives 32.36 dB for UBS and 33.63 dB for \dbs{}, a 1.27 dB improvement on the
same joint-RGB metric as the main table. The ordering holds on all six scenes,
consistent with the 1.26 dB gain at batch size one.

\FloatBarrier
\section{Conditioning and End-to-End Efficiency}
\label{app:efficiency-controls}

\textbf{Inference controls.}
Table~\ref{tab:cached_controls} compares fused CUDA kernels with inputs already
materialized. Cached N-DGS stores query precision, regression, and conditional
covariance; the factorized-query control uses
$\boldsymbol{\Sigma}_{qq}=\mathbf{R}\mathbf{S}^2\mathbf{R}^{\top}$ to obtain
the inverse analytically, while still deriving regression and conditional
covariance. Direct conditioning is 2.7--2.8$\times$ faster than on-the-fly
N-DGS, 2.0--2.3$\times$ faster than its cached form, and 1.8$\times$ faster
than the factorized-query control. These timings exclude production input
preparation and therefore differ from the production-slicing timings in
Table~\ref{tab:comparison_overview}. Persistent caching applies to inference;
training updates require refreshing the cached quantities.

\begin{table}[!htbp]
\centering
\caption{Slicing time (ms) for one million primitives on an A100. Medians of
100 trials after 20 warmups; inputs are materialized before timing.}
\label{tab:cached_controls}
\begin{tabular}{lrr}
\toprule
Parameterization & $C=3$ & $C=4$ \\
\midrule
N-DGS, on-the-fly & 0.224 & 0.269 \\
N-DGS, cached & 0.178 & 0.197 \\
N-DGS, factorized query & 0.141 & 0.180 \\
\dgs{}, direct & \textbf{0.079} & \textbf{0.100} \\
\bottomrule
\end{tabular}
\vspace{-1em}
\end{table}

\textbf{Frame-time decomposition.}
Slicing is a vertex-like preprocessing stage before projection, tile binning,
sorting, and rasterization. If its baseline fraction is $f$ and its speedup is
$s$, an unchanged remainder would give the Amdahl speedup
$[(1-f)+f/s]^{-1}$. In Table~\ref{tab:frame_decomposition}, slicing accounts
for 42--64\% of baseline frame time. The downstream code is shared, but the
learned primitives have different projected footprints and tile workloads:
the measured remainder is 1.10--1.68$\times$ slower for \dgs{}. The observed
end-to-end gains are consequently 1.25--1.91$\times$. These per-scene frame
times differ from arithmetic-mean FPS across all nine Mip-NeRF 360 scenes
(Table~\ref{tab:mcmc_comparison}).

\begin{table}[!htbp]
\centering
\caption{Frame-time decomposition (ms) on trained checkpoints using matched
training caps. Each pair lists N-DGS / \dgs{}; the remainder includes all
post-slicing stages.}
\label{tab:frame_decomposition}
\setlength{\tabcolsep}{5pt}
\begin{tabular}{lrrrr}
\toprule
Scene & Slicing & Remainder & Total & Speedup \\
\midrule
\texttt{lego} & 1.56 / 0.22 & 2.19 / 2.42 & 3.75 / 2.65 & 1.42$\times$ \\
\texttt{counter} & 7.38 / 0.57 & 8.77 / 12.18 & 16.15 / 12.75 & 1.27$\times$ \\
\texttt{garden} & 24.83 / 1.69 & 13.96 / 18.61 & 38.79 / 20.30 & 1.91$\times$ \\
\texttt{bicycle} & 29.78 / 2.02 & 24.71 / 41.59 & 54.49 / 43.61 & 1.25$\times$ \\
\bottomrule
\end{tabular}
\vspace{-1em}
\end{table}
\FloatBarrier

\section{Limitations and Societal Impact}
\label{sec:limitations}

\textbf{Limitations.}
The evaluation covers Gaussian and Beta kernels. Larger-scale scenes may
change memory, visibility, and partitioning bottlenecks; other kernels and
domains require validation.

\textbf{Societal impact.}
Faster neural rendering can benefit simulation, telepresence, robotics, and content creation, but can also facilitate unauthorized reconstruction of private spaces. This work introduces no new data sources or user-facing decision systems.

\textbf{Existing assets.}
All datasets and baseline methods used in this work are public research assets and are cited in the paper. We use NeRF Synthetic, Mip-NeRF 360, D-NeRF, 6DGS-PBR, and 7DGS-PBR only for academic evaluation and do not redistribute the datasets or baseline code.

\newpage
\section*{NeurIPS Paper Checklist}

\begin{enumerate}

\item {\bf Claims}
    \item[] Question: Do the main claims made in the abstract and introduction accurately reflect the paper's contributions and scope?
    \item[] Answer: \answerYes{}
    \item[] Justification: The abstract and introduction state the method, matched-pair scope, speedups, and quality claims. The experiments and appendix report the corresponding benchmark results, while limitations and scope are discussed in the conclusion and supplemental limitations section.

\item {\bf Limitations}
    \item[] Question: Does the paper discuss the limitations of the work performed by the authors?
    \item[] Answer: \answerYes{}
    \item[] Justification: The conclusion includes future work, and the supplemental Section~\ref{sec:limitations} discusses evaluation scope, lack of city-scale/outdoor driving-scale validation, and the need to validate additional kernels and domains.

\item {\bf Theory assumptions and proofs}
    \item[] Question: For each theoretical result, does the paper provide the full set of assumptions and a complete (and correct) proof?
    \item[] Answer: \answerYes{}
    \item[] Justification: Appendices~\ref{sec:conditional_coordinates} and~\ref{sec:spatial_scaling_proof} state assumptions and prove the conditional-coordinate correspondence and Gaussian displacement bound.

    \item {\bf Experimental result reproducibility}
    \item[] Question: Does the paper fully disclose all the information needed to reproduce the main experimental results of the paper to the extent that it affects the main claims and/or conclusions of the paper (regardless of whether the code and data are provided or not)?
    \item[] Answer: \answerYes{}
    \item[] Justification: The experimental setup specifies datasets, metrics, evaluation protocol, optimizer, iteration count, densification protocol, batch-size differences for dynamic MCMC results, initialization, hardware, and CUDA/rasterizer implementation details.

\item {\bf Open access to data and code}
    \item[] Question: Does the paper provide open access to the data and code, with sufficient instructions to faithfully reproduce the main experimental results, as described in supplemental material?
    \item[] Answer: \answerNo{}
    \item[] Justification: The experiments use cited public benchmark datasets, and the paper provides implementation details. Public code and complete reproduction scripts are not provided.

\item {\bf Experimental setting/details}
    \item[] Question: Does the paper specify all the training and test details (e.g., data splits, hyperparameters, how they were chosen, type of optimizer) necessary to understand the results?
    \item[] Answer: \answerYes{}
    \item[] Justification: The experimental setup describes the datasets, evaluation protocol, metrics, training pipeline, optimizer, iteration count, densification settings, batch size for dynamic MCMC runs, and parameter initialization. Additional standard-densification details and per-scene results are provided in the appendix.

\item {\bf Experiment statistical significance}
    \item[] Question: Does the paper report error bars suitably and correctly defined or other appropriate information about the statistical significance of the experiments?
    \item[] Answer: \answerYes{}
    \item[] Justification: Appendix~\ref{app:robustness} reports means and sample standard deviations across three matched seeds for the NeRF Synthetic dBS/UBS comparison and coupling ablation. Other benchmark results are single-run scene averages; the repeated runs do not establish variability across all five datasets.

\item {\bf Experiments compute resources}
    \item[] Question: For each experiment, does the paper provide sufficient information on the computer resources (type of compute workers, memory, time of execution) needed to reproduce the experiments?
    \item[] Answer: \answerYes{}
    \item[] Justification: The paper states that experiments use a single NVIDIA A100 GPU and reports training time, Gaussian count, and rendering speeds in the main and appendix tables. Kernel benchmarks also report slicing-only forward-pass timing on 1M Gaussians.
    
\item {\bf Code of ethics}
    \item[] Question: Does the research conducted in the paper conform, in every respect, with the NeurIPS Code of Ethics \url{https://neurips.cc/public/EthicsGuidelines}?
    \item[] Answer: \answerYes{}
    \item[] Justification: The work uses public benchmark datasets and does not involve human subjects, private data collection, or deployment in user-facing decision systems. The paper also discusses societal impact in the supplement.

\item {\bf Broader impacts}
    \item[] Question: Does the paper discuss both potential positive societal impacts and negative societal impacts of the work performed?
    \item[] Answer: \answerYes{}
    \item[] Justification: Supplemental Section~\ref{sec:limitations} discusses beneficial uses such as simulation, telepresence, robotics, and content creation, as well as potential misuse such as unauthorized reconstruction of private spaces.
    
\item {\bf Safeguards}
    \item[] Question: Does the paper describe safeguards that have been put in place for responsible release of data or models that have a high risk for misuse (e.g., pre-trained language models, image generators, or scraped datasets)?
    \item[] Answer: \answerNA{}
    \item[] Justification: The paper does not release a pretrained generative model, scraped dataset, or other high-risk asset requiring controlled release safeguards.

\item {\bf Licenses for existing assets}
    \item[] Question: Are the creators or original owners of assets (e.g., code, data, models), used in the paper, properly credited and are the license and terms of use explicitly mentioned and properly respected?
    \item[] Answer: \answerNo{}
    \item[] Justification: The original datasets, baseline methods, and rasterizer are cited, and Section~\ref{sec:limitations} describes their academic use without redistribution. The paper does not itemize the license and terms of use of every asset.

\item {\bf New assets}
    \item[] Question: Are new assets introduced in the paper well documented and is the documentation provided alongside the assets?
    \item[] Answer: \answerNA{}
    \item[] Justification: The paper introduces a method and describes its CUDA implementation, but does not release a new dataset or pretrained model.

\item {\bf Crowdsourcing and research with human subjects}
    \item[] Question: For crowdsourcing experiments and research with human subjects, does the paper include the full text of instructions given to participants and screenshots, if applicable, as well as details about compensation (if any)? 
    \item[] Answer: \answerNA{}
    \item[] Justification: The paper does not involve crowdsourcing or research with human subjects.

\item {\bf Institutional review board (IRB) approvals or equivalent for research with human subjects}
    \item[] Question: Does the paper describe potential risks incurred by study participants, whether such risks were disclosed to the subjects, and whether Institutional Review Board (IRB) approvals (or an equivalent approval/review based on the requirements of your country or institution) were obtained?
    \item[] Answer: \answerNA{}
    \item[] Justification: The paper does not involve crowdsourcing or human-subject experiments, so IRB approval is not applicable.

\item {\bf Declaration of LLM usage}
    \item[] Question: Does the paper describe the usage of LLMs if it is an important, original, or non-standard component of the core methods in this research? Note that if the LLM is used only for writing, editing, or formatting purposes and does \emph{not} impact the core methodology, scientific rigor, or originality of the research, declaration is not required.
    \item[] Answer: \answerNA{}
    \item[] Justification: LLMs are not a component of the proposed representation, conditioning equations, training objective, or evaluation method.

\end{enumerate}

\end{document}